\documentclass[twocolumn, aps,superscriptaddress, groupedaddress]{revtex4}  
\usepackage[paperwidth=210mm,paperheight=297mm,centering,hmargin=2cm,vmargin=2cm]{geometry}

\usepackage{titlesec}
\titlespacing*{\section}{0pt}{1.1\baselineskip}{\baselineskip}

\usepackage{mathtools}
\usepackage{lipsum}

\usepackage{amsmath,amsthm,amsfonts,amssymb}  
\usepackage{tcolorbox}
\usepackage{fontenc}
\usepackage[all]{xy}
\usepackage{graphicx,subfigure,multirow,comment}
\usepackage{tikz}
\usepackage[normalem]{ulem}
\usetikzlibrary{shapes, arrows, shadows} 
\usepackage[latin1]{inputenc}
\usepackage[scaled]{helvet}
\usepackage[toc, page]{appendix}
\usepackage{graphicx}				
\usepackage{amssymb}

\makeatletter
\newtheorem*{rep@theorem}{\rep@title}
\newcommand{\newreptheorem}[2]{%
\newenvironment{rep#1}[1]{%
 \def\rep@title{#2 \ref{##1}}%
 \begin{rep@theorem}}%
 {\end{rep@theorem}}}
\makeatother

\newtheorem{theorem}{Theorem}
\newreptheorem{theorem}{Theorem}
\newtheorem{lemma}{Lemma}

\newtheorem{define}{Definition}
\newreptheorem{define}{Definition}

\newcommand\numberthis{\addtocounter{equation}{1}\tag{\theequation}}

\def\BibTeX{{\rm B\kern-.05em{\sc i\kern-.025em b}\kern-.08emT\kern-.1667em\lower.7ex\hbox{E}\kern-.125emX}}
    
\begin{document}

%
% The "title" command has an optional parameter, allowing the author to define a "short title" to be used in page headers.

\title{Counterfactual diagnosis}

%
% The "author" command and its associated commands are used to define the authors and their affiliations.
% used to denote shared contribution to the research.

\author{Jonathan G. Richens}
\affiliation{Babylon Health, London, United Kingdom}
\email{jonathan.richens@babylonhealth.com}

\author{Ciar{\'a}n M. Lee}
\affiliation{Babylon Health, London, United Kingdom}
\affiliation{University College London, United Kingdom}

\author{Saurabh Johri}
\affiliation{Babylon Health, London, United Kingdom}

%
% This command processes the author and affiliation and title information and builds
% the first part of the formatted document.

%Causal knowledge is crucial for effective reasoning in many areas of science and medicine. This is especially true in medical diagnosis, where a clinician aims to explain a patients symptoms by determining the most likely underlying diseases causing them. Despite this, all previous approaches to machine learning assisted %\mariaedits{Maybe change to automated diagnosis rather than ML diagnosis} diagnosis, including both model-based Bayesian and Deep Learning approaches, fail to adequately incorporate causal knowledge in the diagnosis procedure. In this work, we propose a new diagnostic algorithm based on counterfactual inference, which captures the causal aspect of diagnosis overlooked by previous approaches. We test our algorithm against a cohort of 44 doctors, using a state-of-the-art diagnostic model. %Whilst the baseline posterior inference algorithm \mariaedits{baseline posterior inference algorithm $\rightarrow$ the baseline algorithm, which is based on the ranking of diseases according to the posterior marginal probabilities}, achieves diagnostic accuracy on par the average doctor, our counterfactual algorithm achieves expert clinical accuracy, placing it in the top fifth of doctors.

%However, all previous approaches to Machine-Learning assisted diagnosis, including Deep Learning and model-based Bayesian approaches, do not distinguish correlation from causation.

\begin{abstract}

Machine learning promises to revolutionize clinical decision making and diagnosis. In medical diagnosis a doctor aims to explain a patient's symptoms by determining the diseases \emph{causing} them. However, existing diagnostic algorithms are purely associative, identifying diseases that are strongly correlated with a patients symptoms and medical history. We show that this inability to disentangle correlation from causation can result in sub-optimal or dangerous diagnoses. To overcome this, we reformulate diagnosis as a counterfactual inference task and derive new counterfactual diagnostic algorithms. We show that this approach is closer to the diagnostic reasoning of clinicians and significantly improves the accuracy and safety of the resulting diagnoses. We compare our counterfactual algorithm to the standard Bayesian diagnostic algorithm and a cohort of 44 doctors using a test set of clinical vignettes. While the Bayesian algorithm achieves an accuracy comparable to the average doctor, placing in the top 48\% of doctors in our cohort, our counterfactual algorithm places in the top 25\% of doctors, achieving expert clinical accuracy. This improvement is achieved simply by changing how we query our model, without requiring any additional model improvements. Our results show that counterfactual reasoning is a vital missing ingredient for applying machine learning to medical diagnosis.

\end{abstract}

\maketitle

\section{Introduction}

%Providing effective and accessible primary care is a fundamental challenge for global healthcare systems. Over half the global population has no access to primary healthcare \cite{hogan2018monitoring}, and the prevalence of diagnostic errors in primary care has been recognised by the World Health Organisation as a high priority problem \cite{cresswell2013global}. In the US alone an estimated 5\% of outpatients receive the wrong diagnosis every year \cite{singh2014frequency,singh2017global}. These errors are particularly common when diagnosing patients with serious medical conditions, with an estimated 20\% of these patients being misdiagnosed at the level of primary care \cite{graber2013incidence} and one in three of these misdiagnoses resulting in serious patient harm \cite{singh2013types,singh2014frequency}.

Providing accurate and accessible diagnoses is a fundamental challenge for global healthcare systems. In the US alone an estimated 5\% of outpatients receive the wrong diagnosis every year \cite{singh2014frequency,singh2017global}. These errors are particularly common when diagnosing patients with serious medical conditions, with an estimated 20\% of these patients being misdiagnosed at the level of primary care \cite{graber2013incidence} and one in three of these misdiagnoses resulting in serious patient harm \cite{singh2013types,singh2014frequency}.

In recent years, artificial intelligence and machine learning have emerged as powerful tools for solving complex problems in diverse domains \cite{silver2017mastering,brown2019superhuman,tomavsev2019clinically}. In particular, machine learning assisted diagnosis promises to revolutionise healthcare by leveraging abundant patient data to provide precise and personalised diagnoses \cite{liang2019evaluation,topol2019high,de2018clinically,yu2018artificial,jiang2017artificial}. Despite significant research efforts and renewed commercial interest, diagnostic algorithms have struggled to achieve the accuracy of doctors in differential diagnosis \cite{semigran2016comparison,miller1986internist,shwe1991probabilistic,miller2010history, heckerman1992toward1,heckerman1992toward,razzaki2018comparative}, where there are multiple possible causes of a patients symptoms.

This raises the question, why do existing approaches struggle with differential diagnosis? All existing diagnostic algorithms, including Bayesian model-based and Deep Learning approaches, rely on \emph{associative} inference---they identify diseases based on how correlated they are with a patients symptoms and medical history. This is in contrast to how doctors perform diagnosis, selecting the diseases which offer the best \emph{causal} explanations for the patients symptoms. As noted by Pearl, associative inference is the simplest in a hierarchy of possible inference schemes \cite{pearl2018theoretical}. Counterfactual inference sits at the top of this hierarchy, and allows one to ascribe causal explanations to data. Here, we argue that diagnosis is fundamentally a counterfactual inference task. We show that failure to disentangle correlation from causation places strong constraints on the accuracy of associative diagnostic algorithms, sometimes resulting in sub-optimal or dangerous diagnoses. To resolve this, we present a causal definition of diagnosis that is closer to the decision making of clinicians, and derive new counterfactual diagnostic algorithms to validate this approach. 
% something like "when providing a diagnostic differential including 5 candidate diseases, the counterfactual algorithm reduced the number of misdiagnoses (missing the model disease) by 30%."

We compare the accuracy of our counterfactual algorithms to a state-of-the-art associative diagnostic algorithm and a cohort of 44 doctors, using a test set of 1671 clinical vignettes. In our experiments, the doctors achieve an average diagnostic accuracy of $71.40\%$, while the associative algorithm achieves a similar accuracy of $72.52\%$, placing in the top $48\%$ of doctors in our cohort. However, our counterfactual algorithm achieves an average accuracy of $77.26\%$, placing in the top $25\%$ of the cohort and achieving expert clinical accuracy. These improvements are particularly pronounced for rare diseases, where diagnostic errors are more common and often more serious. We find that the counterfactual algorithm offers a greatly improved diagnostic accuracy for rare and very rare diseases (29.2\% and 32.9\% of cases respectively) compared to the associative algorithm. 

Importantly, the counterfactual algorithm achieves these improvements using the same disease model as the associative algorithm---only the method for querying the model has changed. This backwards compatibility is particularly important as disease models require significant resources to learn \cite{miller2010history}. Our algorithms can thus be applied as an immediate upgrade to existing Bayesian diagnostic models, even those outside of medicine \cite{cai2017bayesian,yongli2006bayesian,dey2005bayesian,cai2014multi}. \\

\section{Diagnosis}
 
Here, we outline the basic principles and assumptions underlying the current approach to algorithmic diagnosis. We then detail scenarios where this approach breaks down due to causal confounding, and propose a set of principles for designing new diagnostic algorithms that overcome these pitfalls. Finally, we use these principles to propose two new diagnostic algorithms based on the notions of necessary and sufficient causation. 

\vspace{-5mm}

\subsection{Associative diagnosis}\label{section: associative diagnosis}

Since its formal definition \cite{reiter1987theory}, model-based diagnosis has been synonymous with the task of using a model $\theta $ to estimate the likelihood of a fault component $D$ given findings $\mathcal E$ \cite{de1990using},

\begin{equation}\label{posterior}
    P(D|\mathcal E; \, \theta)
\end{equation}

In medical diagnosis $D$ represents a disease or diseases, and findings $\mathcal E$ can include symptoms, tests outcomes and relevant medical history. In the case of diagnosing over multiple possible diseases, e.g. in a differential diagnosis, potential diseases are ranked in terms of their posterior. Model based diagnostic algorithms are either discriminative, directly modelling the conditional distribution of diseases $D$ given input features $\mathcal E$ \eqref{posterior}, or generative, modelling the prior distribution of diseases and findings and using Bayes rule to estimate the posterior,

%directly model the conditional distribution of class labels given input signals

\begin{equation}\label{bayes rule}
P(D|\mathcal E; \, \theta) = \frac{P(\mathcal E|D ; \, \theta)P(D ; \, \theta)}{P(\mathcal E ; \, \theta)}
\end{equation}

Examples of discriminative diagnostic models include neural network and deep learning models \cite{de2018clinically,litjens2016deep,liu2014early,wang2016deep,liang2019evaluation}, whereas generative models are typically Bayesian networks  \cite{kahn1997construction,cai2017bayesian,miller1986internist,shwe1991probabilistic,heckerman1992toward1,heckerman1992toward,morris2001recognition}. 

How does this approach compare to how doctors perform diagnosis? It has long been argued that diagnosis is the process of finding causal explanations for a patient's symptoms \cite{stanley2013logic,thagard2000scientists,qiu1989models,cournoyea2014causal,kirmayer2004explaining,ledley1959reasoning,westmeyer1975diagnostic,rizzi1994causal,benzi2011medical,patil1981causal,davis1990diagnosis}. For example, \cite{stanley2013logic} concludes ``The generation of hypotheses is by habitual abduction. The physician relies on her knowledge of possible causes that explain the symptoms''. Likewise \cite{merriam1995merriam} defines diagnosis as ``the investigation or analysis of the cause or nature of a condition, situation, or problem''. That is, given the evidence presented by the patient, a doctor attempts to determine the diseases that are the best explanation---the \emph{most likely underlying cause}---of the symptoms presented. We propose the following causal definition of diagnosis,

\begin{quote} 
\centering 
\emph{The identification of the diseases that are most likely to be causing the patient's symptoms, given their medical history.}
\end{quote} 

Despite the wealth of literature placing causal  reasoning at the centre of diagnosis, to the best of our knowledge there are no existing approaches to model-based diagnosis that employ modern causal analysis techniques \cite{pearl2009causality,halpern2016actual}.

\begin{figure}[h!]
    \centering
    \includegraphics[scale=0.45]{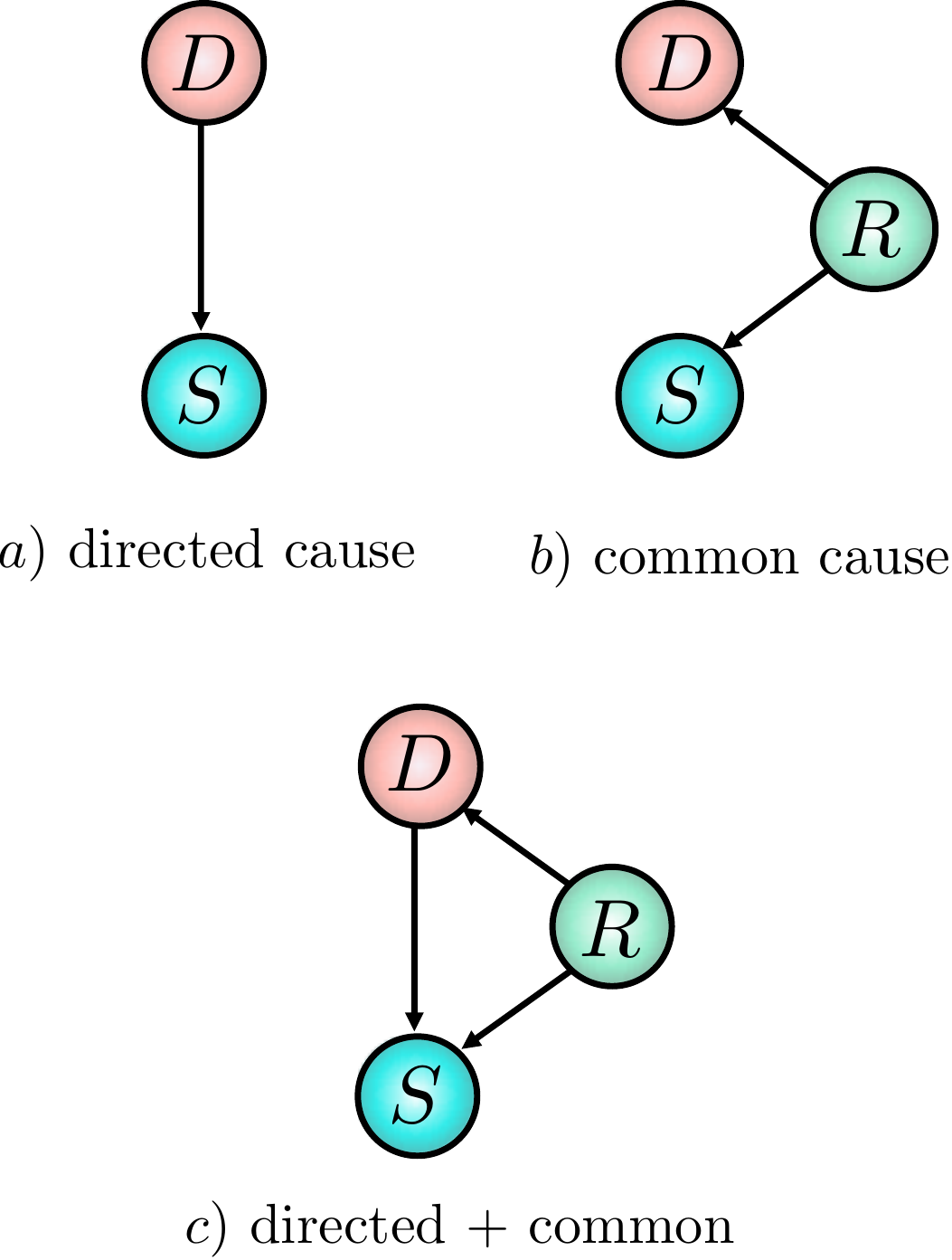}
    \caption{{\footnotesize a) Disease $D$ is a direct cause of symptom $S$, b) $D$ does not cause symptom $S$, but they are correlated by a latent common cause $R$, c) $D$ is a direct cause of $S$ and a latent common cause $R$ is present.}}
    \label{fig:fig2}
\end{figure}

%\vspace{-5mm}

It is well known that using the posterior to identify causal relations can lead to spurious conclusions in all but the simplest causal scenarios---a phenomenon known as confounding \cite{pearl2014comment}. For example, figure \ref{fig:fig2} a) shows a disease $D$ which is a direct cause of a symptom $S$. In this scenario, $D$ is a plausible explanation for $S$, and treating $D$ could alleviate symptom $S$. In figure \ref{fig:fig2} b), variable $R$ is a confounder for $D$ and $S$, for example $R$ could be a genetic factor which increases a patients chance of developing disease $D$ and experiencing symptom $S$. Although $D$ and $S$ can be strongly correlated in this scenario, $P(D = T | S = T) \gg P(D = T)$ (where $D=T$ denotes the presence of $D$), $D$ cannot have caused symptom $S$ and so would not constitute a reasonable diagnosis. In general, diseases are related to symptoms by both directed and common causes that cannot be simply disentangled, as shown in figure \ref{fig:fig2} c). The posterior \eqref{posterior} does not differentiate between these different scenarios and so is insufficient for assigning a diagnosis to a patient's symptoms in all but the simplest of cases, and especially when there are multiple possible causes for a patient's symptoms.

\vspace{3mm}

\begin{tcolorbox}
{\footnotesize
\textbf{Example 1:} An elderly smoker reports chest pain, nausea, and fatigue. A good doctor will present a diagnosis that is both likely and relevant given the evidence (such as angina). Although this patient belongs to a population with a high prevalence of emphysema, this disease is unlikely to have caused the symptoms presented and should not be put forward as a diagnosis. Emphysema is positively correlated with the patient's symptoms, but this is primarily due to common causes \cite{greenland1999causal}.}\\

{\footnotesize \textbf{Example 2:} \cite{cooper2005predicting} found that asthmatic patients who were admitted to hospital for pneumonia were more aggressively treated for the infection, lowering the sub-population mortality rate. An associative model trained on this data to diagnose pneumonia will learn that asthma is a protective risk factor---a dangerous conclusion that could result in a less aggressive treatment regime being proposed for asthmatics, despite the fact that asthma increases the risk of developing pneumonia. In this example, the confounding factor is the unobserved level of care received by the patient. 
}
\end{tcolorbox}
\vspace{5mm}

%Both generative and discriminative diagnostic algorithms are associative, The need for causal modelling when applying machine learning to healthcare has been explored in other ... \cite{ghassemi2018opportunities} .... in the following we propose a new approach to performing diagnosis with generative diagnostic models which overcomes these issues. 

%a neural network trained on electronic health records \cite{liang2019evaluation} will learn correlations between patient features and disease instances and use these to predict diagnoses for , but will not be able to differentiate a .... \cite{marcus2018deep}. The need for causal modelling when applying machine leanring to healthcare has been explored in other ... \cite{ghassemi2018opportunities}

Real world examples of confounding, such as Examples 1 \& 2, have lead to increasing calls for causal knowledge to be properly incorporated into decision support algorithms in healthcare
\cite{ghassemi2018opportunities}.

\subsection{Principles for diagnostic reasoning}\label{sec: principles}

An alternative approach to associative diagnosis is to reason about causal responsibility (or causal attribution)---the probability that the occurrence of the effect $S$ was in fact brought about by target cause $D$ \cite{waldmann2017oxford}. This requires a diagnostic measure $\mathcal M(D, \mathcal E)$ for ranking the likelihood that a disease $D$ is causing a patient's symptoms given evidence $\mathcal E$. We propose the following three minimal desiderata that should be satisfied by any such diagnostic measure, 

\begin{enumerate}
    \item[i)] The likelihood that a disease $D$ is causing a patient's symptoms should be proportional to the posterior likelihood of that disease $\mathcal M(D, \mathcal E)\propto P(D = T |\mathcal E)$ (\emph{consistency}), 
    \item[ii)] A disease $D$ that cannot cause any of the patient's symptoms can not constitute a diagnosis, $\mathcal M(D, \mathcal E) = 0$ (\emph{causality}),
    \item[iii)] Diseases that explain a greater number of the patient's symptoms should be more likely %ranked highly -
    (\emph{simplicity}).
\end{enumerate}

The justification for these desiderata is as follows. Desideratum i) states that the likelihood that a disease explains the patient's symptoms is proportional to the likelihood that the patient has the disease in the first place. Desideratum ii) states that if there is no causal mechanism whereby disease $D$ could have generated any of the patient's symptoms (directly or indirectly), then $D$ cannot constitute causal explanation of the symptoms and should be disregarded. Desideratum iii) incorporates the principle of Occam's razor---favouring simple diagnoses with few diseases that can explain many of the symptoms presented. Note that the posterior only satisfies the first desiderata, violating the last two.

\subsection{Counterfactual diagnosis} \label{section: counterfcatual diagnosis}

To quantify the likelihood that a disease is causing the patient's symptoms, we employ counterfactual inference \cite{shpitser2009effects,morgan2015counterfactuals,pearl2009causal}. Counterfactuals can test whether certain outcomes \emph{would have} occurred had some precondition been different. Given evidence $\mathcal E = e$ we calculate the likelihood that we would have observed a different outcome $\mathcal E = e'$, \emph{counter to the fact} $\mathcal E = e$, had some hypothetical intervention taken place. The counterfactual likelihood is written $P( \mathcal E = e' \, | \, \mathcal E = e, \text{do}(X\! = \!x))$ where do$(X\! = \!x)$ denotes the intervention that sets variable $X$ to the value $X=x$, as defined by Pearl's calculus of interventions \cite{pearl2009causality} (see appendix \ref{appendix: twin networks} for formal definitions). 

Counterfactuals provide us with the language to quantify how well a disease hypothesis $D=T$ explains symptom evidence $S = T$ by determining the likelihood that the symptom would not be present if we were to intervene and `cure' the disease by setting $\text{do}(D=F)$, given by the counterfactual probability $ P( S = F \, | \, S = T, \text{do}(D = F))$. If this probability is high, $D = T$ constitutes a good causal explanation of the symptom.  Note that this probability refers to two contradictory states of $S$ and so cannot be represented as a standard posterior \cite{peters2017elements, pearl2009causality}. In appendix \ref{appendix: twin networks} we describe how these counterfactual probabilities are calculated.

Inspired by this example, we propose two counterfactual diagnostic measures, which we term the \emph{expected disablement} and \emph{expected sufficiency}. We show in Theorem 1 at the end of this section that both measures satisfy all three desiderata from section~\ref{sec: principles}.

\begin{define}[Expected disablement]\label{def: expected disablement}
The expected disablement of disease $D$ is the number of present symptoms that we would expect to switch off if we intervened to cure $D$,

\begin{equation}\label{def exp dis}
        \mathbb E_\text{dis}(D, \mathcal E) := \sum\limits_{ \mathcal S'}\!\left|\mathcal S_+ \setminus \mathcal S_+' \right|P(\mathcal S' | \mathcal E, \text{do}(D = F))
\end{equation}

\noindent where $\mathcal E$ is the factual evidence and $\mathcal S_+$ is the set of factual positively evidenced symptoms. The summation is calculated over all possible counterfactual symptom evidence states $\mathcal S'$ and $\mathcal S_+'$ denotes the positively evidenced symptoms in the counterfactual symptom state. $\text{do}(D = F)$ denotes the counterfactual intervention setting $D \rightarrow F$. $\left|\mathcal S_+ \setminus \mathcal S_+' \right|$ denotes the cardinality of the set of symptoms that are present in the factual symptom evidence but are not present in the counterfactual symptom evidence.

\end{define}

The expected disablement derives from the notion of necessary cause \cite{halpern2016actual}, whereby $D$ is a necessary cause of $S$ if $S=T$ if and only if $D=T$. The expected disablement therefore captures how well disease $D$ alone can explain the patient's symptoms, as well as the likelihood that treating $D$ alone will alleviate the patient's symptoms.

\begin{define}[expected sufficiency]
\label{def expected sufficiency}
The expected sufficiency of disease $D$ is the number of positively evidenced symptoms we would expect to persist if we intervene to switch off all other possible causes of the patient's symptoms,

\begin{equation}\label{expected sufficiency}
    \mathbb E_\text{suff}(D, \mathcal E) := \sum\limits_{\mathcal S'}\!\left|\mathcal S_+' \right|P(\mathcal S' | \mathcal E, \text{do}(\mathsf{Pa}(\mathcal S_+) \setminus D = F))
\end{equation}

\noindent where the summation is over all possible counterfactual symptom evidence states $\mathcal S'$ and $\mathcal S_+'$ denotes the positively evidenced symptoms in the counterfactual symptom state.  $\mathsf{Pa}(\mathcal S_+) \setminus D $ denotes the set of all direct causes of the set of positively evidenced symptoms excluding disease $D$, and $\text{do}(\mathsf{Pa}(\mathcal S_+) \setminus D = F)$ denotes the counterfactual intervention setting all $\mathsf{Pa}(\mathcal S_+'\setminus D)\rightarrow F$. $\mathcal E$ denotes the set of all factual evidence. $\left|\mathcal S_+' \right|$ denotes the cardinality of the set of present symptoms in the counterfactual symptom evidence.
\end{define}

The expected sufficiency derives from the notion of sufficient cause \cite{halpern2016actual}, whereby $D$ is a sufficient cause of $S$ if the presence of $D$ implies the subsequent occurrence of $S$ but, as $S$ can have multiple causes, the presence of $S$ does not imply the prior occurrence of $D$. Typically, diseases are sufficient causes of symptoms. By performing counterfactual interventions to remove all possible causes of the symptoms (both diseases and exogenous influences), the only remaining cause is $D$ and so we isolate its effect as a sufficient cause in our model. 

\begin{theorem}[Diagnostic properties of expected disablement and expected sufficiency]

Expected disablement and expected sufficiency satisfy the three desiderata from section~\ref{sec: principles}.
\end{theorem}
\noindent The proof is provided in appendices \ref{appendix: properties} and \ref{appendix: properties exp}. 

\section{Methods}\label{sect: methods}
Here, we introduce the statistical disease models we use to test the diagnostic measures outlined in the previous sections. We then derive simplified expressions for the expected disablement and sufficiency in these models.

\subsection{Structural causal models for diagnosis} \label{section: Structural causal models for diagnosis}

%The disease models we consider are probabilistic graphical models (BNs), or Bayesian Networks, that model the relationships between hundreds of diseases, risk factors (the causes of diseases), and symptoms, as binary nodes that are either on (true) or off (false). We denote true and false with the standard integer notation 1 and 0 respectively. A BN is specified by a directed acyclic graph (DAG) and a joint probability distribution over all nodes which factorises with respect to the graph structure. A simple example is shown in Fig.~\ref{3 layer BN} (a), which depicts a BN whose graphical structure describes a three layer network. These models consist of a top layer of risk factor nodes $R_i$, a middle layer of disease nodes $D_j$, and a bottom layer of symptoms $S_k$. As the joint distribution factorises with respect to the graph structure, it is specified by the risk factor priors $P(R_i)$, the disease conditional probability distributions (sometimes referred to as conditional probabilitiy tables, or CPTs) $P(D_j|R_1, \ldots, R_n)$, and the symptom conditional distributions $P(S_k|D_1, \ldots, D_m)$. Three-layer BNs of this form comprise the current state of the art disease models \cite{razzaki2018comparative}. Whilst our results will be derived for these models, they can be simply extended to models with more or less complicated dependencies \cite{shwe1991probabilistic,heckerman1992toward2}.

The disease models we use in our experiments are Bayesian Networks (BNs) that model the relationships between hundreds of diseases, risk factors and symptoms. BNs are widely employed as diagnostic models as they are interpretable \footnote{In this context, a model or algorithm is interpretable if it is possible to determine why the algorithm has reached a given diagnosis} and explicitly encode causal relations between variables---a prerequisite for causal and counterfactual analysis \cite{pearl2009causality}. These models typically represent diseases, symptoms and risk-factors as binary nodes that are either on (true) or off (false). We denote true and false with the standard integer notation 1 and 0 respectively. 

A BN is specified by a directed acyclic graph (DAG) and a joint probability distribution over all nodes which factorises with respect to the DAG structure. If there is a directed arrow from node $X$ to $Y$, then $X$ is said to be a \emph{parent} of $Y$, and $Y$ to be a \emph{child} of $X$. A node $Z$ is said to be an \emph{ancestor} of $Y$ if there is a directed path from $Z$ to $Y$. A simple example BN is shown in Fig. \ref{3 layer BN} (a), which depicts a BN modeling diseases, symptoms, and risk factors (the causes of diseases).

BN disease models have a long history going back to the INTERNIST-1 \cite{miller1986internist}, Quick Medical Reference (QMR) \cite{shwe1991probabilistic,miller2010history}, and PATHFINDER \cite{heckerman1992toward1,heckerman1992toward} systems, with many of the original systems corresponding to noisy-OR networks with only disease and symptom nodes, known as BN2O networks \cite{morris2001recognition}. Recently, three-layer BNs as depicted in Fig.~\ref{3 layer BN} (a) have replaced these two layer models \cite{razzaki2018comparative}. These models make fewer independence assumptions and allow for disease risk-factors to be included. Whilst our results will be derived for these models, they can be simply extended to models with more or less complicated dependencies \cite{shwe1991probabilistic,heckerman1992toward2}.

% Backup

%Disease models have a long history going back to the INTERNIST-1 \cite{miller1986internist}, Quick Medical Reference (QMR) \cite{shwe1991probabilistic,miller2010history}, and PATHFINDER \cite{heckerman1992toward1,heckerman1992toward} systems, with many of the original systems corresponded to noisy-OR networks with only disease and symptom nodes, known as BN2O networks \cite{morris2001recognition}. However, three-layer BNs of the form described in the previous paragraph and depicted in Fig.~\ref{3 layer BN} (a) comprise the current state of the art disease models \cite{razzaki2018comparative}. Whilst our results will be derived for these models, they can be simply extended to models with more or less complicated dependencies \cite{shwe1991probabilistic,heckerman1992toward2}.

\begin{figure}
    \centering
    \includegraphics[scale=0.4]{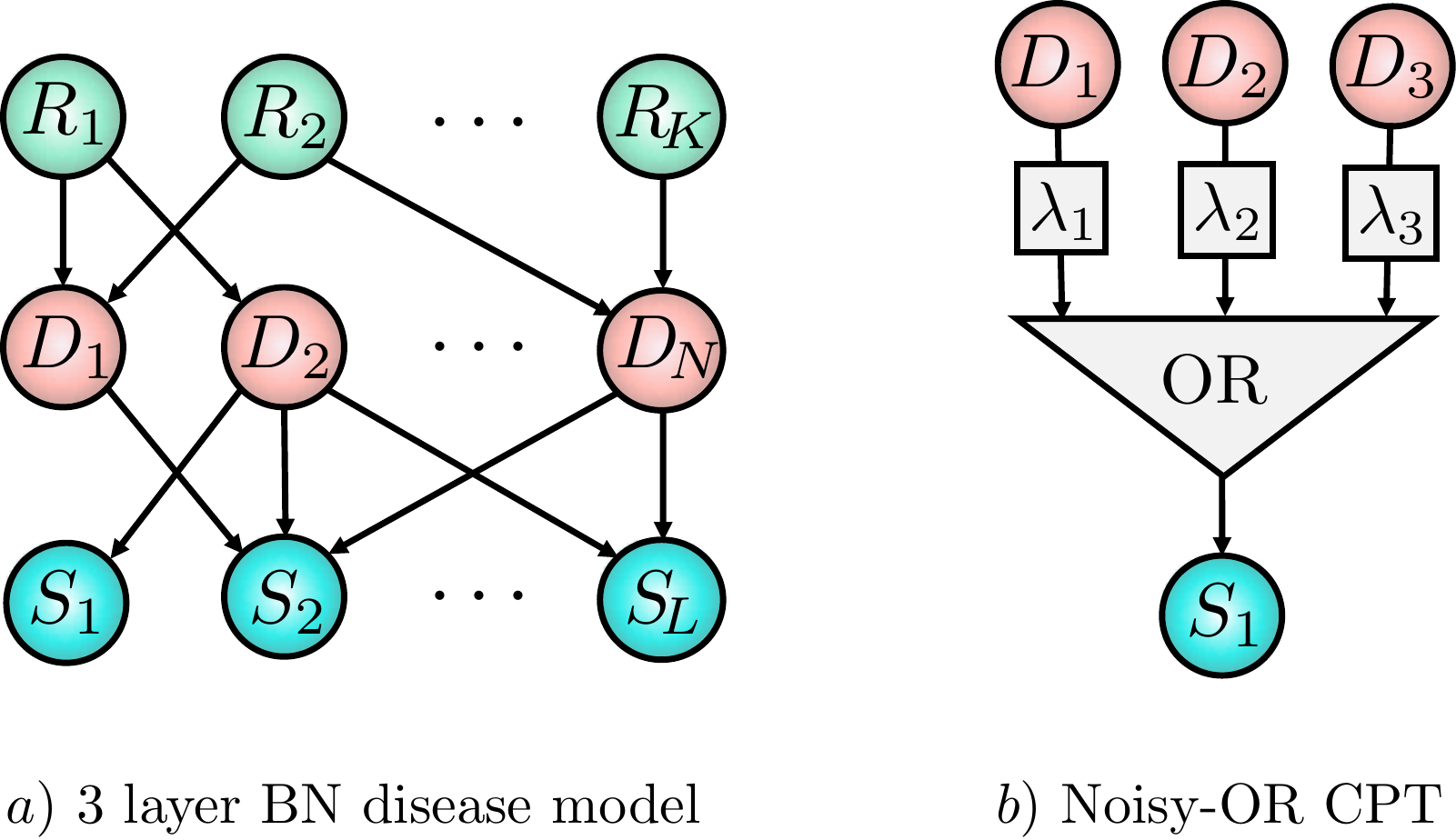}
    \caption{\footnotesize a) Three layer Bayesian network representing risk factors $R_i$, diseases $D_j$ and symptoms $S_k$. b) noisy-OR CPT. $S$ is the Boolean OR function of its parents, each with an independent probability $\lambda_i$ of being ignored, removing them from the OR function.}
    \label{3 layer BN}
\end{figure}

In the field of causal inference, BNs are replaced by the more fundamental Structural Causal Models (SCMs), also referred to as Functional Causal Models and Structural Equation Models \cite{peters2017elements, lee2017causal}. SCMs are widely applied and studied, and their relation to other approaches, such as probabilistic graphical models and BNs, is well understood \cite{lauritzen1996graphical,pearl2009causality}. The key characteristic of SCMs is that they represent each variable as deterministic functions of their direct causes together with an unobserved exogenous `noise' term, which itself represents all causes outside of our model. That the state of the noise term is unknown induces a probability distribution over observed variables. For each variable $Y$, with parents in the model $X$, there is a noise term $u_Y$, with unknown distribution $P(u_Y)$ such that $Y=f(x,u_Y)$ and $P(Y = y | X = x) = \sum_{u_Y : f(x, u_Y) = y} P(U_Y = u_Y)$. By incorporating knowledge of the functional dependencies between variables, SCMs enable us to determine the response of variables to interventions (such as treatments). As we show in section~\ref{section: diagnostic twin networks}, existing diagnostic BNs such as BN2O networks \cite{morris2001recognition} are naturally represented as SCMs. %structural causal models.

\subsection{Noisy-OR twin diagnostic networks}\label{section: noisy or twin networks} \label{section: diagnostic twin networks}
When constructing disease models it is common to make additional modelling assumptions beyond those implied by the DAG structure. The most widely used of these correspond to `noisy-OR' models \cite{shwe1991probabilistic}. Noisy-OR models are routinely used for modelling in medicine, as they reflect basic intuitions about how diseases and symptoms are related \cite{nikovski2000constructing,rish2002accuracy}. Additionally they support efficient inference \cite{heckerman1990tractable} and learning \cite{halpern2013unsupervised,arora2017provable}, and allow for large BNs to be described by a number of parameters that grows linearly with the size of the network \cite{onisko2001learning,halpern2013unsupervised}. %We now derive expressions for the expected disablement and expected sufficiency for these models, which allow these measures to be determined using standard inference techniques. 

Under the noisy-OR assumption, a parent $D_i$ activates its child $S$ (causing $S=1$) if i) the parent is on, $D_i = 1$, and ii) the activation does not randomly fail. The probability of failure, conventionally denoted as $\lambda_{D_i, S}$, is independent from all other model parameters. The `OR' component of the noisy-OR states that the child is activated if \emph{any} of its parents successfully activate it. Concretely, the value $s$ of $S$ is the Boolean OR function $\vee$ of its parents activation functions, $s = \vee_i f(d_i, u_i)$, where the activation functions take the form $f(d_i, u_i) = d_i \wedge \bar u_i$, $\wedge$ denotes the Boolean AND function, $d_i\in \{0, 1\}$ is the state of a given parent $D_i$ and $u_i\in \{0, 1\}$ is a latent noise variable ($\bar u_i := 1-u_i$) with a probability of failure $P(u_i = 1) = \lambda_{D_i, S}$. The noisy-OR model is depicted in Fig 1. b).  Intuitively, the noisy-OR model captures the case where a symptom only requires a single activation to switch it on, and `switching on' a disease will never `switch off' a symptom. For further details on noisy-OR disease modelling see appendix \ref{appendix: scms}.

% Backup
%The noisy-OR assumption is a reasonable assumption for modelling diseases, as we argue below. First, consider the assumption that the generative function is a Boolean OR of the individual parent `activation functions' $x_i\wedge \bar u_i$. This is equivalent to assuming that the activations from parents to children never `destructively interfere'. That is, if $D_i$ is activating symptom $S$, and so is $D_j$, then this joint activation never cancels out to yield $S = 0$. As a consequence, all that is required for a symptom to be present is that at least one disease to be causing it, and likewise for diseases being caused by risk factors. This property of noisy-OR is a natural assumption for diseases modelling. Indeed, diseases are typically---by definition---sufficient causes of their symptoms, and risk factors are defined such that they are sufficient causes of diseases. For further details of the noisy-OR model and how this assumption relates to modelling diseases see appendix \ref{appendix: scms}.

We now derive expressions for the expected disablement and expected sufficiency for these models using twin-networks method for computing counterfactuals introduced in \cite{balke1994counterfactual,shpitser2012counterfactuals}. This method represents real and counterfactual variables together in a single SCM---the twin network---from which counterfactual probabilities can be computed using standard inference techniques. This approach greatly amortizes the inference cost of calculating counterfactuals compared to abduction \cite{pearl2009causality}, which is intractable for large SCMs. We refer to these diagnostic models as twin diagnostic networks, see appendix \ref{appendix: twin networks} for further details. %We now derive expressions for the expected disablement and expected sufficiency for 3-layer noisy-OR disease models in terms of corrections to the standard posterior probabilities. 

\begin{theorem}\label{expected sufficiency theorem}
For 3-layer noisy-OR BNs (formally described in appendices \ref{appendix: scms}-\ref{appendix: twin networks}), the expected sufficiency and expected disablement of disease $D_k$ are given by

%For noisy-OR networks described in appendix \ref{appendix: scms}-\ref{appendix: twin networks}, the expected sufficiency and expected disablement of disease $D_k$ are given by 

\begin{comment}
\begin{equation}\label{theorem 1 equation}
    %\mathbb E_\text{suff}(D_k,\mathcal E) =
    \frac{\sum\limits_{\mathcal S\subseteq \mathcal S_+}|\mathcal S_+\setminus \mathcal S|P(\mathcal S_- = 0, \mathcal S^{\setminus k} = 1, D_k = 1| \mathcal R)\tau(\mathcal S, D_k)}{P(\mathcal S_\pm |\mathcal R)}
\end{equation}

where, for the expected sufficiency,

\begin{equation}
\tau(\mathcal S, D_k) = \prod\limits_{s\in \mathcal S_+\setminus \mathcal Z}(1-\lambda_{k, s})\prod\limits_{s\in \mathcal S}\lambda_{k, s}
\end{equation}

and for the expected disablement.

\end{comment}

\begin{equation}\label{result 0}
    %\mathbb E(D_k,\mathcal E) =
    \frac{\sum\limits_{\mathcal Z\subseteq \mathcal S_+}(-1)^{|\mathcal Z|}P(\mathcal S_- \!=\! 0, \mathcal Z \!=\! 0, D_k \!=\! 1|\mathcal R)\tau (k, \mathcal  Z)}{P(\mathcal S_\pm |\mathcal R)}
\end{equation}

where for the expected sufficiency

\begin{equation}
    \tau (k, \mathcal  Z) =\sum\limits_{S\in \mathcal S_+\setminus \mathcal Z}(1-\lambda_{D_k, S})
\end{equation}

and for the expected disablement

\begin{equation}
    \tau (k, \mathcal  Z) =\sum\limits_{S\in \mathcal Z}\left(1-\frac{1}{\lambda_{D_k, S}}\right)
\end{equation}

where $\mathcal S_\pm$ denotes the positive and negative symptom evidence, $\mathcal R$ denotes the risk-factor evidence, and $\lambda_{D_k, S}$ is the noise parameter for $D_k$ and $S$. 

\end{theorem}

The proof is provided by theorem \ref{theorem deriv exp suff} in appendix \ref{appendix: exp suff} and by theorem \ref{theorem deriv expected dis} in appendix  \ref{appendix: expected dis}. Note that \eqref{result 0} recovers the standard posterior $P(D_k = 1 |\mathcal E)$ in the limit that $\tau(D_k, \mathcal Z) \rightarrow 1$ $\forall$ $ \mathcal Z$.  %These $\tau(D_k, \mathcal Z)$ constitute simple counterfactual corrections to the posterior, which in \eqref{result 0} is presented in the same form as in the quickscore algorithm \cite{heckerman1990tractable}, and hence theorem \ref{expected sufficiency theorem} can be seen as a counterfactual quickscore algorithm \footnote{Unlike the quickscore algorithm, we do not assume independent diseases. Quickscore algorithm detailed in \cite{heckerman1990tractable}.}.

\begin{comment}

We now show that the expected disablement satisfies our conditions for diagnostic measures. Firstly, \eqref{noisy or exp dis theorem} satisfies simplicity as it is upper bounded by $|\mathcal S_+\wedge \text{Ch}(D_k)|$ where $\text{Ch}(D_k)$ is the set of positively evidenced symptoms that are children of $D_k$. The consistency criterion is also satisfied as every posterior in \eqref{noisy or exp dis theorem} is proportional to $P(D_k = 1| \mathcal C = 0, \mathcal S_\pm\setminus \mathcal C, \mathcal R)$, hense $\lim\limits_{P(D_k = 1| \mathcal E)\rightarrow 0}\mathbb E(D_k,\mathcal E) = 0$. \\

Finally, note that every term in \eqref{noisy or exp dis theorem} has a coefficient of the form $\prod\limits_{S\in\mathcal C}\left(\lambda_{D_k, S}-1 \right)$. This is identical to 0 in the case that $\exists S \in \mathcal C$ s.t. $S\not \in  \text{Ch}(D_k)$, which in the noisy-OR model is equivalent to $\lambda_{D_k,S}=1$. Hence if $D_k$ has no positively evidenced children then $\mathbb E_{D_k, \mathcal E} = 0$, satisfying the causality condition. 
\end{comment}

\section{Results} \label{section: experiments}

Here we outline our experiments comparing the expected disablement and sufficiency to posterior inference using the models outlined in the previous section. We introduce our test set which includes a set of clinical vignettes and a cohort of doctors. We then evaluate our algorithms across several diagnostic tasks.

%MAYBE CHANGE TO EXPERIMENTS. Describe vignettes like albert does. 

\subsection{Diagnostic model and datasets}

One approach to validating diagnostic algorithms is to use electronic health records (EHRs) \cite{liang2019evaluation,topol2019high,de2018clinically,yu2018artificial,jiang2017artificial}. A key limitation of this approach is the difficulty in defining the ground truth diagnosis, where diagnostic errors result in mislabeled data. This problem is particularly pronounced for differential diagnoses because of the large number of candidate diseases and hence diagnostic labels, incomplete or inaccurate recording of case data, high diagnostic uncertainty and ambiguity, and biases such as the training and experience of the clinician who performed the diagnosis. 

To resolve these issues, a standard method for assessing doctors is through the examination of simulated diagnostic cases or \emph{clinical vignettes} \cite{peabody2004measuring}. A clinical vignette simulates a typical patient's presentation of a disease, containing a non-exhaustive list of evidence including symptoms, medical history, and basic demographic information such as age and birth gender \cite{razzaki2018comparative}. This approach is often more robust to errors and biases than real data sets such as EHRs, as the task of simulating a disease given its known properties is simpler than performing a differential diagnosis, and has been found to be effective for evaluating human doctors \cite{peabody2004measuring,veloski2005clinical,converse2015methods,dresselhaus2004evaluation} and comparing the accuracy of doctors to symptom checker algorithms \cite{semigran2015evaluation,semigran2016comparison,razzaki2018comparative,middleton2016sorting}.

We use a test set of 1671 clinical vignettes, generated by a separate panel of doctors qualified at least to the level of general practitioner \footnote{equivalent to board certified primary care physicians}. Where possible, symptoms and risk factors match those in our statistical disease model. However, to avoid biasing our study the vignettes include any additional clinical information as case notes, which are available to the doctors in our experiments. Each vignette is authored by a single doctor and then verified by multiple doctors to ensure that it represents a realistic diagnostic case. For each vignette the true disease is masked and the algorithm returns a diagnosis in the form of a full ranking of all modeled diseases using the vignette evidence. The disease ranking is computed using the posterior for the associative algorithm, and the expected disablement or expected sufficiency for the counterfactual algorithms. Doctors provide an independent differential diagnosis in the form of a partially ranked list of candidate diseases.

%Each vignette represents a realistic presentation of a patient with a single disease or condition, containing a non-exhaustive list of evidence including symptoms, medical history, and basic demographic information such as age and birth gender \cite{razzaki2018comparative}. Where possible, symptoms and risk factors matches those in our statistical disease model. However, to avoid biasing our study the vignettes include any additional clinical information as case notes, which are available to the doctors in our experiments. Each vignette is authored by a single doctor and then verified by multiple doctors to ensure that it represents a realistic diagnostic case.

In all experiments the counterfactual and associative algorithms use identical disease models to ensure that any difference in diagnostic accuracy is due to the ranking query used. The disease model used is a three layer noisy-OR diagnostic BN as described in sections \ref{sect: methods} and appendix \ref{appendix: scms}. The BN is parameterised by a team of doctors and epidemiologists \cite{razzaki2018comparative,middleton2016sorting}. The prior probabilities of diseases and risk factors are obtained from epidemiological data, and conditional probabilities are obtained through elicitation from multiple independent medical sources and doctors \footnote{2: It should be noted that the disease model evaluated in the following experiments is not the current production model used for the purposes of diagnosis and triage by Babylon Health$^\text{TM}$. This article is for general information and academic purposes, and this disease model is used to facilitate discussion on this topic. This article is not designed to be relied upon for any other
purpose.}. The expected disablement and expected sufficiency are calculated using Theorem \ref{expected sufficiency theorem}.

\subsection{Counterfactual v.s associative rankings}

Our first experiment compares the diagnostic accuracy of ranking diseases using the posterior \eqref{posterior}, expected disablement and expected sufficiency \eqref{result 0}. For each of the 1671 vignettes the top-$k$ ranked diseases are computed, with $k = 1, \ldots 20$, and the top-$k$ accuracy is calculated as fraction of the 1671 diagnostic vignettes where the true disease is present in the $k$-top ranking. The results are presented in figure \ref{fig:top_n}. The expected disablement and expected sufficiency give almost identical accuracies for all $k$ on our test set, and for the sake of clarity we present the results for the expected sufficiency alone. A complete table of results is present in Appendix \ref{appendix: experimental}.

\begin{figure}[h!]
    \centering
    \includegraphics[scale=0.22]{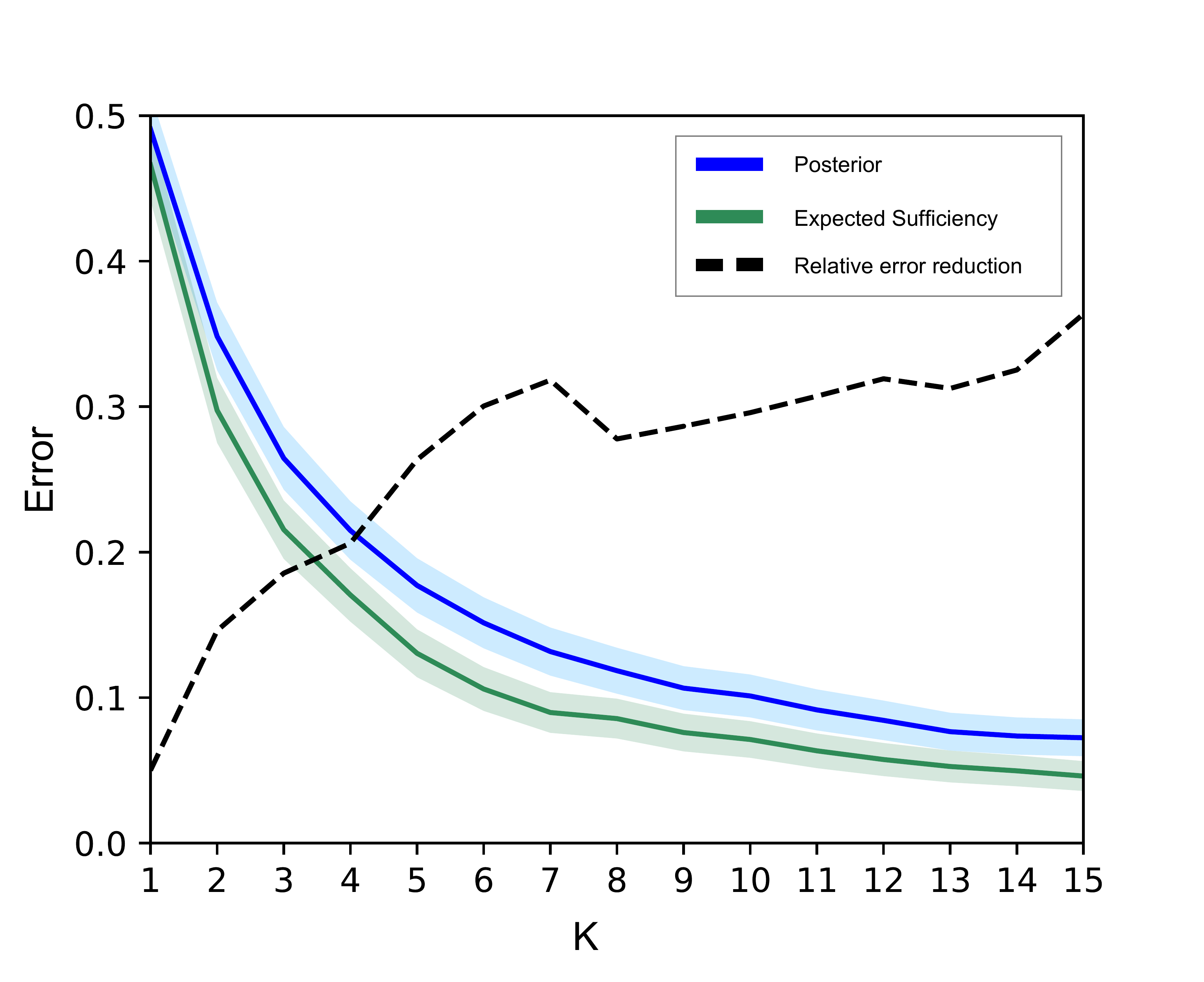}
    \caption{\footnotesize \textbf{Top $k$ accuracy of Bayesian and counterfactual algorithms}. Figure shows the top $k$ error (1 - accuracy) of the counterfactual (green line) and associative (blue line) algorithms over all 1671 vignettes v.s $k$. Shaded regions give 95\% confidence intervals. The black dashed line shows the relative reduction in error when switching from the associative to counterfactual algorithm, given by $1-e_c / e_a$ where $e_a$ is the error rate of the associative algorithm, and $e_c$ is the error rate of the counterfactual algorithm. Results shown for $k = 1, \ldots 15$, for complete results see Appendix \ref{appendix: experimental}.}
    \label{fig:top_n}
\end{figure}

For $k = 1$, returning the top ranked disease, the counterfactual algorithm achieves a 2.5\% higher accuracy than the associative algorithm. For $k > 1$ the performance of the two algorithms diverge, with the counterfactual algorithm giving a large reduction in the error rate over the associative algorithm. For $k>5$, the counterfactual algorithm reduces the number of misdiagnoses by approximately 30\% compared to the associative algorithm. This suggests that the best candidate disease is reasonably well identified by the posterior, but the counterfactual ranking is significantly better at identifying the next most likely diseases. These secondary candidate diseases are especially important in differential diagnosis for the purposes of triage and determining optimal testing and treatment strategies. 

%The fact that the improvement is sustained for large $k$ suggests that the counterfactual algorithm is removing spurious diseases from the ranking. 

\begin{table}[!h]
\footnotesize
\begin{center}
\begin{tabular}{| c | c | c | c | c | c | c |}
\hline
\textbf{} & \multicolumn{6}{ c |}{Vignettes}  \\ 
\cline{2-7}
 & All & VCommon & Common & Uncommon & Rare & VRare \\
\hline
N & 1671 & 131 & 413 & 546 & 353 & 210  \\ \hline
Mean (A)  & 3.81   & 2.85  & 2.71 & 3.72 & 4.35  & 5.45   \\ \hline
Mean (C)  & 3.16   & 2.5  & 2.32  & 3.01  & 3.72  & 4.38 \\ \hline
Wins (A) & 31 &  2  &  7 & 9 & 9 & 4 \\ \hline
Wins (C)  & 412 & 20 & 80 & 135 & 103 & 69\\ \hline
Draws & 1228 & 131 & 326 & 402 & 241 & 137 \\ \hline
\end{tabular}

\end{center}
\caption{\footnotesize \textbf{Position of true disease in ranking stratified by rareness of disease}. Table shows the mean position of the true disease for the associative (A) and counterfactual (C) algorithms. The results for expected disablement are almost identical to the expected sufficiency and are included in the appendices. Results are stratified over the rareness of the disease (given the age and gender of the patient), where VCommon = Very common and VRare = very rare, and All is over all 1671 vignettes regardless of disease rarity. N is the number of vignettes belonging to each rareness category. Mean(X) is the average position of the true disease for algorithm X. Wins (X) is the number of vignettes where algorithm X ranked the true disease higher than its counterpart, and Draws is the number of vignettes where the two algorithms ranked the true disease in the same position. For full results including uncertainties see appendix \ref{appendix: experimental}.}
\label{tab:topn counts main} 
\end{table}

A simple method for comparing two rankings is to compare the position of the true disease in the rankings. Across all 1671 vignettes we found that the counterfactual algorithm ranked the true disease higher than the associative algorithm in 24.7\% of vignettes, and lower in only 1.9\% of vignettes. On average the true disease is ranked in position 3.16$\pm$4.4 by the counterfactual algorithm, a substantial improvement over 3.81$\pm$5.25 for the associative algorithm (see Table \ref{tab:topn counts main}). 

In table \ref{tab:topn counts main} we stratify the vignettes by the prior incidence rates of the true disease by very common, common, uncommon, rare and very rare. While the counterfactual algorithm achieves significant improvements over the associative algorithm for both common and rare diseases, the improvement is particularly large for rare and very rare diseases, achieving a higher ranking for $29.2\%$ and $32.9\%$ of these vignettes respectively. This improvement is important as rare diseases are typically harder to diagnose and include many serious conditions where diagnostic errors have the greatest consequences. 

%A possible explanation for this improvement is that very rare diseases have relatively low posterior probabilities, whilst there are many common diseases with high priors and little or no causal connection to the patients evidence. These common diseases will often have a higher posterior than the rare diseases, unless very strong evidence for the rare disease is reported. For example, a $5\%$ likelihood of having lung cancer should be included in a diagnosis, but this can be ignored by the posterior ranking if there are many common diseases with prior likelihoods higher than $5\%$. 

\subsection{Comparing to doctors}

Our second experiment compares the counterfactual and associative algorithms to a cohort of 44 doctors. Each doctor is assigned a set of at least 50 vignettes (average 159), and returns an independent diagnosis for each vignette in the form of a partially ranked list of $k$ diseases, where the size of the list $k$ is chosen by the doctor on a case-by-case basis (average diagnosis size is $2.58$ diseases). For a given doctor, and for each vignette diagnosed by the doctor, the associative and counterfactuals algorithms are supplied with the same evidence (excluding the free text case description) and each returns a top-$k$ diagnosis, where $k$ is the size of the diagnosis provided by the doctor. Matching the precision of the doctor for every vignette allows us to compare the accuracy of the doctor and the algorithms without constraining the doctors to give a fixed number of diseases for each diagnosis. This is important as doctors will naturally vary the size $k$ of their diagnosis to reflect their uncertainty in the diagnostic vignette.%, or if they believe a unlikely but serious disease should be included in the diagnosis.

\begin{figure}[h!]
    \centering
    \includegraphics[scale=0.25]{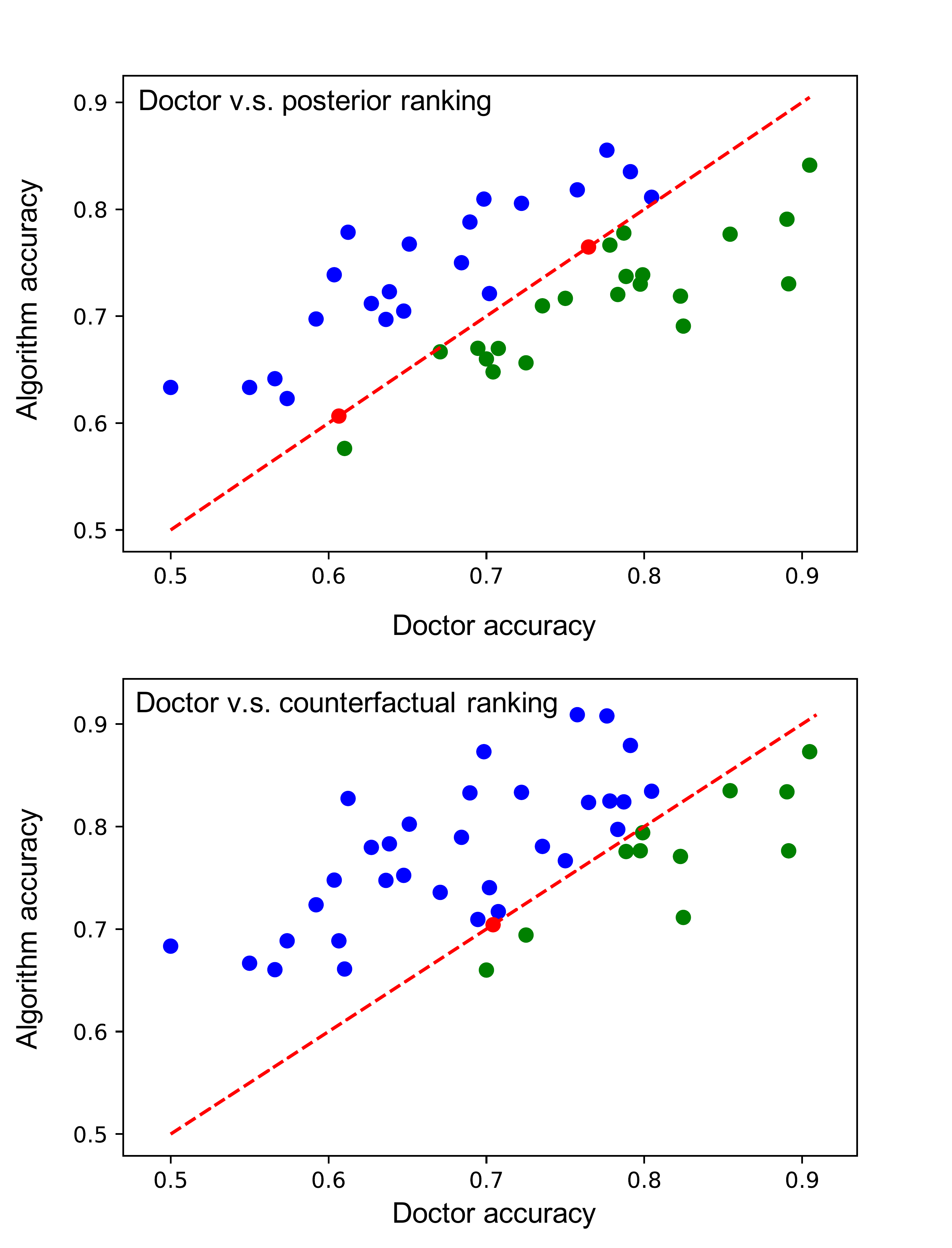}
    \caption{\footnotesize \textbf{Mean accuracy of each doctor compared to Bayesian and counterfactual algorithms.} Figure shows the mean accuracy for each of the 44 doctors, compared to the posterior ranking (top) and expected sufficiency ranking (bottom) algorithms. The line $y=x$ gives a reference for comparing the accuracy of each doctor to the algorithm shadowing them. Points above the line correspond to doctors who achieved a lower accuracy than the algorithm (blue), points on the line are doctors that achieved the same accuracy as the algorithm (red), and below the line are doctors that achieved higher accuracy than the algorithm (green). The linear correlation can be explained by the variation in the difficulty of the sets of vignettes diagnosed by each doctor. Sets of easier/harder vignettes results in higher/lower doctor and algorithm accuracy scores. As the results for the expected disablement and expected sufficiency are almost identical, we show only the results for the expected sufficiency. Complete results are listed in appendix \ref{appendix: experimental}.}
    %Different experiments receive different sets of cases of varying difficulty, hence the accuracy varies between experiments and there is a correlation between the agents accuracies. As the results for the expected disablement and expected sufficiency are practically identical, we show only the results where the counterfactual measure is the expected sufficiency}
    \label{fig:scatter_plots}
\end{figure}

%The linear correlation can be explained by the variation in the difficulty of the sets of vignettes diagnosed by each doctor. Sets of easier/harder vignettes results in higher/lower doctor and algorithm accuracy scores

\begin{table}
\begin{center}
\footnotesize
\begin{tabular}{|c|c|c|c|c|c|} 
\hline
Agent & Accuracy (\%) & $N_{\geq D}$ & $N_{\geq A}$ & $N_{\geq C1}$ & $N_{\geq C2}$\\
\hline
D & 71.40 $\pm$ 3.01 & - & 23 (8) & 12 (4) & 13 (5) \\ 
A & 72.52 $\pm$ 2.97 & 23 (9) & - & 1 (0) &  1 (0)\\
C1 &  77.26 $\pm$ 2.79 & 33 (20) & 44 (13) & - & 36   (0)\\ 
C2 & 77.22 $\pm$ 2.79 & 33 (19) & 44 (14)& 32 (0) & -\\
\hline
\end{tabular}
\end{center}

\caption{ \footnotesize \textbf{Group mean accuracy of doctors and algorithms}. The mean accuracy of the doctors D, associative A and counterfactual algorithms (C1 = expected sufficiency, C2 =  expected disablement), averaged over all experiments. $N_{\geq K}$ gives the number of trials (one for each doctor) where this agent achieved a mean accuracy the same or higher than the mean accuracy of agent $K\in \{\text{D, A, C1, C2}\}$. The bracketed term is the number of trials where the agent scored the same or higher accuracy than agent $K$ to 95\%  confidence, determined by a one sided binomial test.\label{tab:res}
}
\end{table}
%\captionof{table}{{\small The mean accuracy of the doctors $D$, associative $A$ and counterfactual $C$ algorithms, averaged over all experiments. $P_{\geq K}$ gives the probability that, for a randomly selected experiment, the agent scored a higher accuracy than its counterpart agent $K\in \{D, O, C\}$}. The bracketed term is the probability that a randomly selected agent has a higher accuracy than its counterpart with 95\% confidence.}\label{tab:res}

The complete results for each of the 44 doctors, and for the posterior, expected disablement, and expected sufficiency ranking algorithms are included in Appendix \ref{appendix: experimental}. Figure \ref{fig:scatter_plots} compares the accuracy of each doctor to the associative and counterfactual algorithms. Each point gives the average accuracy for one of the 44 doctors, calculated as the proportion of vignettes diagnosed by the doctor where the true disease is  included in the doctor's differential. This is plotted against the accuracy that the corresponding algorithm achieved when diagnosing the same vignettes and returning differentials of the same size as that doctor. 

%There are roughly two types of performance for the doctors and algorithms, depending on the difficulty of the vignettes included in the case set. 
Doctors tend to achieve higher accuracies in case sets involving simpler vignettes---identified by high doctor and algorithm accuracies. Conversely, the algorithm tends to achieve higher accuracy than the doctors for more challenging vignettes---identified by low doctor and algorithm accuracies. This suggests that the diagnostic algorithms are complimentary to the doctors, with the algorithm performing better on vignettes where doctor error is more common and vice versa.

Overall, the associative algorithm performs on par with the average doctor, achieving a mean accuracy across all trails of $72.52\pm 2.97\%$ v.s $71.40\pm 3.01\%$ for doctors. The algorithm scores higher than 21 of the doctors, draws with 2 of the doctors, and scores lower than 21 of the doctors. The counterfactual algorithm achieves a mean accuracy of $77.26\pm 2.79\%$, considerably higher than the average doctor and the associative algorithm, placing it in the top 25\% of doctors in the cohort. The counterfactual algorithm scores higher than 32 of the doctors, draws with 1, and scores a lower accuracy than 12.

In summary, we find that the counterfactual algorithm achieves a substantially higher diagnostic accuracy than the associative algorithm. We find the improvement is particularly pronounced for rare diseases. Whilst the associative algorithm performs on par with the average doctor, the counterfactual algorithm places in the upper quartile of doctors.

\bigskip 

\section{Discussion} \label{section: conclusions}

%\begin{enumerate}
%    \item we have shown that causal and counterfactual reasoning are missing ingredients for...
%    \item we have found evidence that incorporating causal reasoning (via counterfactual inference) can greatly improve the accuracy of diagnostic algorithms. 
%    \item we have shown that it is feasible to achieve better-than-average differential diagnoses using diagnostic models, but to do so required that.... hmm.
%    \item These improvements come for free... can be applied as an immediate upgrade to many existing Bayesian diagnostic algorithms. 
%\end{enumerate}

Poor access to primary healthcare and errors in differential diagnoses represent a significant challenge to global healthcare systems \cite{higgs2008clinical,graber2013incidence,singh2014frequency,singh2013types,liberman2018symptom,singh2017global}. If machine learning is to help overcome these challenges, it is important that we first understand how diagnosis is performed and clearly define the desired output of our algorithms. Existing approaches have conflated diagnosis with associative inference. Whilst the former involves determining the underlying cause of a patient's symptoms, the latter involves learning correlations between patient data and disease occurrences, determining the most likely diseases in the population that the patient belongs to. Whilst this approach is perhaps sufficient for simple causal scenarios involving single diseases, it places strong constraints on the accuracy of these algorithms when applied to differential diagnosis, where a clinician chooses from multiple competing disease hypotheses. Overcoming these constraints requires that we fundamentally rethink how we define diagnosis and how we design diagnostic algorithms.

We have argued that diagnosis is fundamentally a counterfactual inference task and presented a new causal definition of diagnosis. We have derived two counterfactual diagnostic measures, expected disablement and expected sufficiency, and a new class of diagnostic models---twin diagnostic networks---for calculating these measures. Using existing diagnostic models we have demonstrated that ranking disease hypotheses by these counterfactual measures greatly improves diagnostic accuracy compared to standard associative rankings. Whilst the associative algorithm performed on par with the average doctor in our cohort, the counterfactual algorithm places in the top 25\% of doctors in our cohort---achieving expert clinical accuracy. The improvement is particularly pronounced for rare and very rare diseases, where diagnostic errors are typically more common and more serious, with the counterfactual algorithm ranking the true disease higher than the associative algorithm in $29.2\%$ and $32.9\%$ of these cases respectively. Importantly, this improvement comes `for free', without requiring any alterations to the disease model. Because of this backward compatibility our algorithm can be used as an immediate upgrade for existing Bayesian diagnostic algorithms including those outside of the medical setting \cite{cai2017bayesian,yongli2006bayesian,dey2005bayesian,romessis2006bayesian,cai2014multi}.

Whereas other approaches to improving clinical decision systems have focused on developing better model architectures or exploiting new sources of data, our results demonstrate a new path towards expert-level clinical decision systems---changing how we query our models to leverage causal knowledge. Our results add weight to the argument that machine learning methods that fail to incorporate causal reasoning will struggle to surpass the capabilities of human experts in certain domains  \cite{pearl2018theoretical}. Whilst we have focused on comparing our algorithms to doctors, future experiments could determine the effectiveness of these algorithms as clinical support systems---guiding doctors by providing a second opinion diagnosis. Given that our algorithm appears to be complimentary to human doctors, performing better on vignettes that doctors struggle to diagnose, it is likely that the combined diagnosis of doctor and algorithm will be more accurate than either alone. 

\newpage
\bibliographystyle{ieeetr}
\bibliography{sample-base}

% 
% If your work has an appendix, this is the place to put it.

%\appendix

\onecolumngrid

%\section*{Appendices}

\begin{appendices}

The structure of these appendices is as follows. In appendix \ref{appendix: notation} we detail our notation. In appendix \ref{appendix: scms} we outline the tools we use to derive our results -- namely the frameworks of structural causal models (SCMs), introduce noisy-or Bayesian networks, and derive their SCM representation. In appendix \ref{appendix: twin networks} we outline the framework of twin-networks \cite{balke1994counterfactual}, and derive a simplified class of twin networks that we will use for computing our counterfactual diagnostic measures (`twin diagnostic networks'). In appendices \ref{appendix: exp suff} and \ref{appendix: expected dis} we introduce and derive expressions for our counterfactual diagnostic measure s---the expected sufficiency and the expected disablement---for the family of noisy-or diagnostic networks introduced in Appendices \ref{appendix: scms} and \ref{appendix: twin networks}. In appendices \ref{appendix: properties} and \ref{appendix: properties exp} we prove that these two measures satisfy our desiderata. In appendix \ref{appendix: experimental} we list our experimental results.

\section{Notation}\label{appendix: notation}

\textbf{Variables}: For the disease models we consider, all variables $X$ are Bernoulli, $X\in \{0, 1\}$. Where appropriate we refer to $X = 0$ as the variable $X$ being `off', and $X = 1$ as the variable $X$ being `on'. We denote single variables as capital Roman letters, and sets of variables as calligraphic, e.g. $\mathcal X = \{X_1, X_2, \ldots, X_n\}$. The union of two sets of variables $\mathcal X$ and $\mathcal Y$ is denoted $\mathcal X \cup \mathcal Y$, the intersection is denoted $\mathcal X \cap \mathcal Y$, and the relative compliment of $\mathcal X$ w.r.t $\mathcal Y$ as $\mathcal X\setminus \mathcal Y$. The instantiation of a single variable is indicated by a lower case letter, $X = x$, and for a set of variables  $\mathcal X = \underline x$ denotes some arbitrary instantiation of all variables belonging to $\mathcal X$, e.g. $X_1= x_1, X_2 = x_2, \ldots, X_n = x_n$. The probability of $\mathcal X = \underline x$ is denoted $P(\mathcal X = \underline x)$, and sometimes for simplicity is denoted as $P(\underline x)$. \\

For a given variable $X$ and a directed acyclic graph (DAG) $G$, we denote the set of parents of $X$ as $\mathsf{Pa}(X)$, the set of children of $X$ as $\mathsf{Ch}(X)$, all ancestors of $X$ as $\mathsf{Anc}(X)$, and all descendants of $X$ as $\mathsf{Dec}(X)$. If we perform a graph cut operation on $G$, removing a directed edge from $Y$ to $X$, we denote the variable $X$ in the new DAG generated by this cut as $X^{\setminus Y}$.\\

\textbf{Functions}: Bernoulli variables are represented interchangeably as Boolean variables, with $1 \leftrightarrow $ `True' and $0 \leftrightarrow$ `False'. For a given instantiation of a Bernoulli/Boolean variable $X = x$, we denote the negation of $x$ as $\bar x$ -- for example if $x = 1 (0)$, $\bar x = 0 (1)$. We denote the Boolean AND function as $\wedge$, and the Boolean OR function as $\vee$.

\section{structural causal models}\label{appendix: scms}
First we define structural causal models (SCMs), sometimes also called structural equation models or functional causal models. These are widely applied and studied probabilistic models, and their relation to other approaches such as Bayesian networks are well understood \cite{lauritzen1996graphical,pearl2009causality}. The key characteristic of SCMs is that they represent variables as functions of their direct causes, along with an exogenous `noise' variable that is responsible for their randomness.

\begin{define}[Structural Causal Model] \label{functional causal model}
\label{scmdef}
A causal model specifies:
\begin{enumerate}
\item a set of latent, or \emph{noise}, variables $\mathcal U=\{u_1,\dots,u_n\}$, distributed according to $P(\mathcal U)$.
\item a set of observed variables $\mathcal V=\{v_1,\dots, v_n\},$
\item a directed acyclic graph $G$, called the \emph{causal structure} of the model, whose nodes are the variables $\mathcal U\cup \mathcal V$,
\item a collection of functions $F=\{f_1,\dots, f_n\}$, where $f_i$ is a mapping from $\mathcal U \cup \mathcal V/v_i$ to $v_i$. The collection $F$ forms a mapping from $\mathcal U$ to $\mathcal V$. This is symbolically represented as 
$$v_i = f_i(\mathsf{Pa}(v_i), u_i), \text{ for } i=1,\dots, n,$$
where $\text{pa}_i$ denotes the parent nodes of the $i$th observed variable in $G$. % The latent noise term appearing in each $f_i$ can be suppressed into $\text{pa}_i$ by enforcing the convention that every observed node has an independent latent variable as a parent in $G$. This convention is adopted throughout the following.
\end{enumerate}
\end{define}

As the collection of functions $F$ forms a mapping from noise variables $\mathcal U$ to observed variables $\mathcal V$, the distribution over noise variables induces a distribution over observed variables, given by 
\begin{equation} \label{fundamental scm equation}
    P(v_i) := \sum_{u| v_i = f_i(\mathsf{Pa}(v_i), u)} P(u), \text{ for } i=1,\dots, n.
\end{equation}
We can hence assign uncertainty over observed variables despite the the underlying dynamics being deterministic.

In order to formally define a counterfactual query, we must first define the interventional primitive known as the ``$do$-operator'' \cite{pearl2009causality}. Consider a SCM with functions $F$. The effect of intervention $do(X=x)$ in this model corresponds to creating a new SCM with functions $F_{X=x}$, formed by deleting from $F$ all functions $f_i$ corresponding to members of the set $X$ and replacing them with the set of constant functions $X=x$. That is, the $do$-operator forces variables to take certain values, regardless of the original causal mechanism. This represents the operation whereby an agent intervenes on a variable, fixing it to take a certain value. Probabilities involving the $do$-operator, such as $P(Y=y|do(X=x))$, correspond to evaluating ordinary probabilities in the SCM with functions $F_{X=x}$, in this case $P(Y=y)$. Where appropriate, we use the more compact notation of $Y_x$ to denote the variable $Y$ following the intervention do$(X = x)$.\\

Next we define noisy-OR models, a specific class of SCMs for Bernoulli variables that are widely employed as diagnostic models \cite{shwe1991probabilistic,miller2010history,yongli2006bayesian,nikovski2000constructing,heckerman1990tractable,liu2010passive,halpern2013unsupervised,perreault2016noisy,arora2017provable,abdollahi2016unification}. The noisy-OR assumption states that a variable $Y$ is the Boolean OR of its parents $X_1, X_2, \ldots, X_n$, where the inclusion or exclusion of each causal parent in the OR function is decided by an independent probability or `noise' term. The standard approach to defining noisy-OR is to present the conditional independence constraints generated by the noisy-OR assumption \cite{pearl1999probabilities}, 

\begin{equation}\label{noisy-OR def eq}
    P(Y =0\, | \, X_1, \ldots, X_n) = \prod\limits_{i=1}^n P(Y =0\, | \, \text{only}(X_i = 1))
\end{equation}

where $P(Y =0\, | \, \text{only}(X_i = 1))$ is the probability that $Y=0$ conditioned on all of its (endogenous) parents being `off' ($X_j = 0$) except for $X_i$ alone. We denote $P(Y =0\, | \, \text{only}(X_i = 1))=\lambda_{X_i, Y}$ by convention.

The utility of this assumption is that it reduces the number of parameters needed to specify a noisy-OR network to $\mathcal O (N)$ where $N$ is the number of directed edges in the network. All that is needed to specify a noisy-OR network are the single variable marginals $P(X_i = 1)$ and, for each directed edge $X_i \rightarrow Y_j$, a single $\lambda_{X_i, Y_j}$. For this reason, noisy-OR has been a standard assumption in Bayesian diagnostic networks, which are typically large and densely connected and so could not be efficiently learned and stored without additional assumptions on the conditional probabilities. We now define the noisy-OR assumption for SCMs.

\begin{define}[noisy-OR SCM]\label{def noisy scm}
A noisy-OR network is an SCM of Bernoulli variables, where for any variable $Y$ with parents $\mathsf{Pa}(Y) = \{X_1, \ldots , X_N \}$ the following conditions hold

\begin{enumerate}
    \item $Y$ is the Boolean OR of its parents, where for each parent $X_i$ there is a Bernoulli variable $U_i$ whose state determines if we include that parent in the OR function or not
    
    \begin{equation}\label{eq: noisy-OR def}
    y = \bigvee\limits_{i=1}^N \left(x_i \wedge \bar u_i \right) 
    \end{equation}
    i.e. $Y = 1$ if any parent is on, $x_i = 1$, and is not ignored, $u_i = 0$ ($\bar u_i = 1$ where `bar' denotes the negation of $u_i$).
    \item The exogenous latent encodes the likelihood of ignoring the state of each parent in (1), $P(u_Y) = P(u_1,u_2, \ldots, u_N)$. The probability of ignoring the state of a given parent variable is independent of whether you have or have not ignored any of the other parents, $$P(u_1, u_2, \ldots , u_N) = \prod\limits_{i=1}^N P(u_i)$$
    \item For every node $Y$ there is a parent `leak node' $L_Y$ that is singly connected to $Y$ and is always `on', with a probability of ignoring given by $\lambda_{L_Y}$
\end{enumerate}
\end{define} 

The leak node (assumption 3) represents the probability that $Y = 1$, even if $X_i = 0$ $\forall$ $X_i \in \mathsf{Pa}(Y)$. This allows $Y = 1$ to be caused by an exogenous factor (outside of our model). For example, the leak nodes allow us to model the situation that a disease spontaneously occurs, even if all risk factors that we model are absent, or that a symptom occurs but none of the diseases that we model have caused it. It is conventional to treat the leak node associated with a variable $Y$ as a parent node $L_Y$ with $P(L_Y = 1)$. Every variable in the noisy-OR SCM has a single, independent leak node parent. 

Given Definition \ref{def noisy scm}, why is the noisy-or assuption justified for modelling diseases? First, consider the assumption (1), that the generative function is a Boolean OR of the individual parent `activation functions' $x_i\cap \bar u_i$. This is equivalent to assuming that the activations from diseases or risk-factors to their children never `destructively interfere'. That is, if $D_i$ is activating symptom $S$, and so is $D_j$, then this joint activation never cancels out to yield $S = F$. As a consequence, all that is required for a symptom to be present is that at least one disease to be causing it, and likewise for diseases being caused by risk factors. This property of noisy-OR, whereby an individual cause is also a sufficient cause, is a natural assumption for diseases modelling -- where diseases are (typically by definition) sufficient causes of their symptoms, and risk factors are defined such that they are sufficient causes of diseases. For example, if preconditions $R_1 = 1$ and $R_2 = 1$ are needed to cause $D = 1$, then we can represent this as a single risk factor $R = R_1 \wedge R_2$. Assumption 2 states that a given disease (risk factor) has a fixed likelihood of activating a symptom (disease), independent of the presence or absence of any other disease (risk factor). In the noisy-or model, the likelihood that we ignore the state of a parent $X_i$ of variable $Y_i$ is given by 

\begin{equation}
    P(u_i = 1) = \frac{P(Y_i = 0 |\text{ do}(X_i = 1))}{P(Y_i = 0 |\text{ do}(X_i = 0))}
\end{equation}

and so is directly associated with a (causal) relative risk. In the case that child $Y$ has two parents, $X_1$ and $X_2$, noisy-OR assumes that this joint relative risk factorises as  

\begin{align}
    P(u_1=1, u_2=1) &= \frac{P(Y = 0 |\text{ do}(X_1 = 1, X_2 = 1))}{P(Y = 0 |\text{ do}(X_1 = 0, X_2 = 0))} = \frac{P(Y = 0 |\text{ do}(X_1 = 1))}{P(Y = 0 |\text{ do}(X_1 = 0))}\times \frac{P(Y = 0 |\text{ do}(X_2 = 1))}{P(Y = 0 |\text{ do}(X_2 = 0))}\\
    &= P(u_1=1)P(u_2=1)
\end{align}

Whilst it is likely that interactions between causal parents will mean that these relative risks are not always multiplicative, it is assumed to be a good approximation. For example, we assume that the likelihood that a disease fails to activate a symptoms is independent of whether or not any other disease similarly fails to activate that symptom.\\

As noisy-OR models are typically presented as Bayesian networks, the above definition of noisy-OR is non-standard. We now show that the SCM definition yields the Bayesian network definition, \eqref{noisy-OR def eq}. 

\begin{theorem}[noisy-OR CPT]
The conditional probability distribution of a child $Y$ given its parents $\{X_1, \ldots, X_N \}$ and obeying Definition \ref{def noisy scm} is given by 
\begin{equation}\label{def eq noisy or}
    P(Y = 0 \, | \, X_1 = x_1, \ldots, X_n = x_N) = \prod\limits_{i=1}^N \lambda_{X_i, Y}^{x_i}
\end{equation}

where 

\begin{equation}
    \lambda_{X_i, Y} = P(Y = 0 |\text{ only }(X_i = 1)) 
\end{equation}

\end{theorem}
\begin{proof}
For $Y=0$, the negation of $y$, denoted $\bar y$, is given by 

\begin{equation}
    \bar y = \neg \left( \bigvee\limits_{i=1}^N \left(x_i \wedge \bar u_i \right) \right) = \bigwedge\limits_{i=1}^N \left( \bar x_i \vee u_i\right)
\end{equation}

The CPT is calculated from the structural equations by marginalizing over the latents, i.e. we sum over all latent states that yield $Y = 0$. Equivalently, we can marginalize over all exogenous latent states multiplied by the above Boolean function, which is 1 if the condition $Y=0$ is met, and 0 otherwise. 

\begin{align*}
     P(Y = 0 \, | \, X_1 = x_1, \ldots, X_n = x_n) &=\sum\limits_{ u_1}\ldots \sum\limits_{ u_N}\bigwedge\limits_{i=1}^N \left( \bar x_i \vee u_i\right) P(u_Y)\\
    &= \sum\limits_{ u_1}\ldots \sum\limits_{ u_N}\prod\limits_{X_i}\left( \bar x_i \vee u_i\right)\prod\limits_{U_i}P(u_i) \\
    &= \prod\limits_{X_i}\sum_{U_i = u_i}P(u_i)\left( \bar x_i \vee u_i\right)\\
    &= \prod\limits_{X_i}\left[P(u_i = 1) + P(u_i = 0)\bar x_i\right]\\
    &= \prod\limits_{X_i}\left[\lambda_{X_i, Y} + (1-\lambda_{X_i, Y})\bar x_i\right]\\
    &= \prod\limits_{X_i}\lambda_{X_i, Y}^{x_i}\numberthis 
\end{align*}
This is identical to the noisy-OR CPT \eqref{noisy-OR def eq}
\end{proof}

where we denote $\lambda_{X_i, Y} = P(u_i)$. The leak node is included as a parent $X_L$ where $P(X_L = 1) = 1$, and a (typically large) probability of being ignored $\lambda_L$. This node represents the likelihood that $Y$ will be activated by some causal influence outside of the model, and is included to ensure that $P(Y = 1 |\wedge_{i=1}^n (X_i = 0))\neq 0$. As the leak node is always on, its notation can be suppressed and it is standard notation to write the CPT as

\begin{equation}
     P(Y = 0 \, | \, X_1 = x_1, \ldots, X_n = x_n) = \lambda_L\prod\limits_{X_i}\lambda_{X_i, Y}^{x_i}\label{noisy leaky or}
\end{equation}

\section{Twin diagnostic networks}\label{appendix: twin networks}
In this appendix we derive the structure of diagnostic twin networks. First we provide a brief overview to the twin-networks approach to counterfactual inference. See \cite{balke1994counterfactual} and \cite{shpitser2007counterfactuals} for more details on this formalism. First, recalling the definition of the \emph{do} operator from the previous section, we define counterfactuals as follows. 

\begin{define}[Counterfactual]
Let $X$ and $Y$ be two subsets of variables in $V$. The counterfactual sentence ``$Y$ would be $y$ (in situation $U$), had $X$ been $x$,'' is the solution $Y=y$ of the set of equations $F_x$, succinctly denoted $Y_x(U)=y$.
\end{define}
As with observed variables in Definition~\ref{functional causal model}, the latent distribution $P(U)$ allows one to define the probabilities of counterfactual statements in the same manner they are defined for standard probabilities \eqref{fundamental scm equation}. 
\begin{equation}\label{basic counterfactual eq}
  P(Y_x=y) = \sum_{u|Y_x(u)=y} P(u).  
\end{equation}
Reference \cite{pearl2009causality} provides an algorithmic procedure for computing arbitrary counterfactual probabilities for a given SCM. First, the distribution over latents is updated to account for the observed evidence. Second, the $do$-operator is applied, representing the counterfactual intervention. Third, the new causal model created by the application of the $do$-operator in the previous step is combined with the updated latent distribution to compute the counterfactual query. In general, denote $\mathcal E$ as the set of factual evidence. The above can be summarised as,

\begin{enumerate}
    \item (abduction). The distribution of the exogenous latent variables $P(u)$ is updated to obtain $P(u\, |\,  \mathcal E)$
    \item (action). Apply the do-operation to the variables in set $X$, replacing the equations $X_i = f_i (\text{Pa}(x_i), u_i)$ with $X_i = x_i$ $\forall$ $X_i \in X$. 
    \item (prediction). Use the modified model to compute the probability of $Y = y$.
\end{enumerate}

The issue with applying this approach to our large diagnostic models is that the first step, updating the exogenous latents, is in general intractable for models with large tree-width. The twin-networks formalism, introduced in \cite{balke1994counterfactual}, is a method which reduces and amortises the cost of this procedure. Rather than explicitly updating the exogenous latents, performing an intervention, and performing belief propagation on the resulting SCM, twin networks allow us to calculate the counterfactual by performing belief propagation on a single `twin' SCM -- without requiring the expensive abduction step. The twin network is constructed as a composite of two copies of the original SCM where copied variables share their corresponding latents \cite{balke1994counterfactual}. We refer to pairs of copied variables as `dual variables'. Nodes on this twin network can then be merged following simple rules outlined in \cite{shpitser2007counterfactuals}, further reducing the complexity of computing the counterfactual query. We now outline the process of constructing the twin diagnostic network in the case of the two counterfactual queries we are interested in -- those with single counterfactual interventions, and those where all counterfactual variables bar one are intervened on.

\[
\includegraphics[scale=0.75]{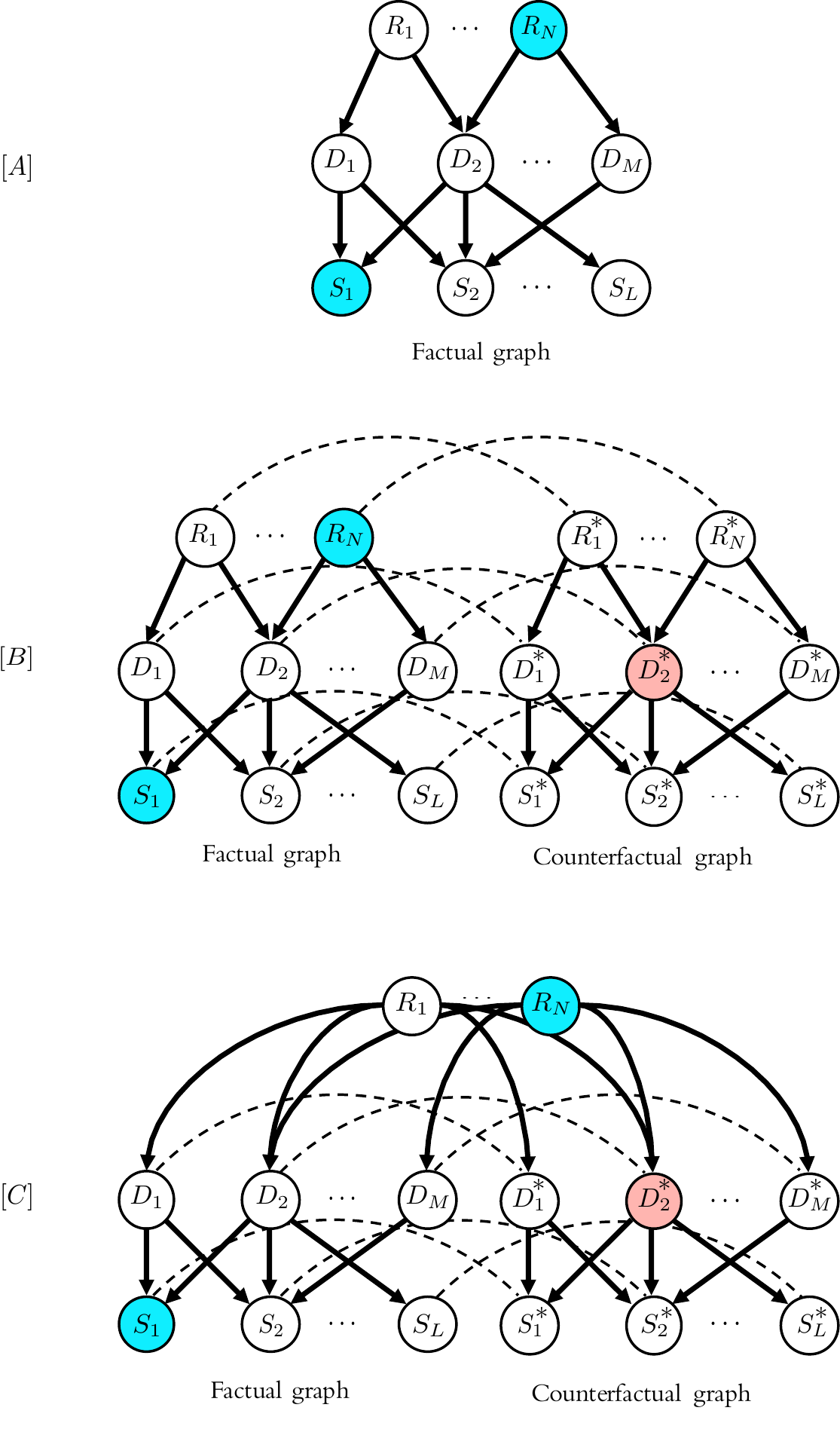}
\]

We assume the DAG structure of our diagnostic model is a three layer network [A]. The top layer nodes represent risk factors, the second layer represent diseases, and the third layer symptoms. We assume no directed edges between nodes belonging to the same layer. To construct the twin network, first the SCM in [A] is copied. In [B] the network on the left will encode the factual evidence in our counterfactual query, and we refer to this as the factual graph. The network on the right in [B] will encode our counterfactual interventions and observations, and we refer to this as the counterfactual graph. We use an asterisk $X^*$ to denote the counterfactual dual variable of $X$.

%\color{red} Under figure B could write real and hypothetical under their respectove graphs. Also is referring to it as hyptothetical instead of counterfactual the best choice? Finally, on figure D and E, should we make arrows from diseases to symptoms straight instead of curved? In figure E it looks like $D_2$ is causing $S_1^\*$, which might confuse a potential reviewer \color{black} \textbf{DONE}

As detailed in \cite{balke1994counterfactual}, the twin network is constructed such that each node on the factual graph shares its exogenous latent with its dual node, so $u_{X_i}^* = u_{X_i}$. These shared exogenous latents are shown as dashed lines in figures [B-E]. First, we consider the case where we perform a counterfactual intervention on a single disease. As shown in [B], we select a disease node in the counterfactual graph to perform our intervention on (in this instance $D_2^*$). In Figure [C], blue circles represent observations and red circles represent interventions. The do-operation severs any directed edges going into $D^*$ and fixes $D^*=0$, as shown in [D] below. 

Once the counterfactual intervention has been applied, it is possible to greatly simplify the twin network graph structure via node merging \cite{shpitser2007counterfactuals}. In SCM's a variable takes a fixed deterministic value given an instantation of all of its parents and its exogenous latent. Hence, if two nodes have identical exogenous latents and parents, they are copies and can be merged into a single node. By convention, when we merge these identical dual nodes we map $X^*\mapsto X$ (dropping the asterisk). Dual nodes which share no ancestors that have been intervened upon can therefore be merged. As we do not perform interventions on the risk factor nodes, all $(R_i, R_i^*)$ are merged (note that for the sake of clarity we do not depict the exogenous latents for risk factors).

\[
\includegraphics[scale=0.7]{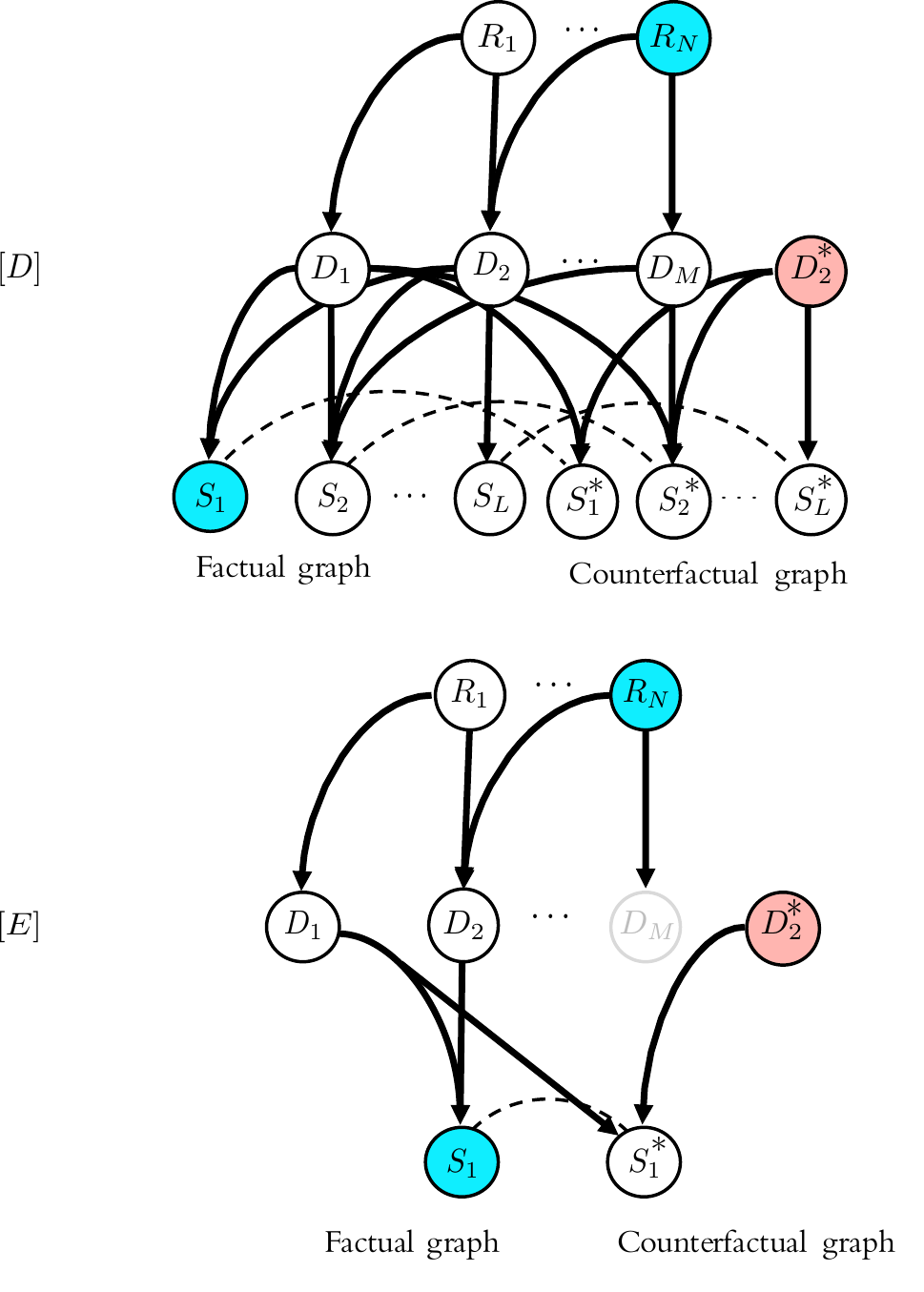}
\]

Next, we merge all dual factual/counterfactual disease nodes that are not intervened on, as their latents and parents are identical (shown in [D]). Finally, any symptoms that are not children of the disease we have intervened on ($D_2$) can be merged, as all of their parent variables are identical. The resulting twin network is shown in [E]. Note that we have also removed any superfluous symptom nodes that are unevidenced, as they are irrelevant for the query. \\ 

In the case that we intervene on all of the counterfactual diseases except one, following the node merging rule outlined above, we arrive at a model with a single disease that is a parent of both factual and counterfactual symptoms, as shown in Figure [F].

\[
\includegraphics[scale = 0.7]{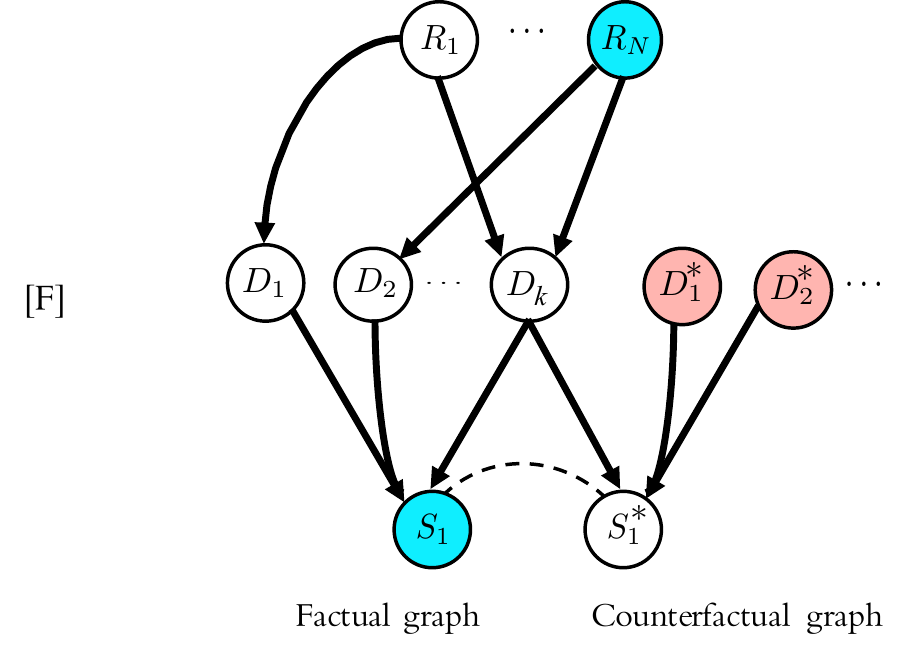}
\]

We refer to the SCMs shown in figures [E] and [F] as `twin diagnostic networks'. The counterfactual queries we are interested in can be determined by applying standard inference techniques such as importance sampling to these models \cite{perov2019multiverse}. 

%For example in [E], if we observe $S_1 = 1$, and we want to determine if $S_1$ would be $0$ if we intervened as $\text{do}(D_2 = 0)$, this counterfactual is identified as $P(S^*_1 = 0 |S_1 = 1, D_2^* = 0)$ on the model shown in Figure [E].

\section{expected sufficiency}\label{appendix: exp suff}

In this appendix we derive a simple closed form expression for our proposed diagnostic measure, the \emph{expected sufficiency}, which corresponds to the case where we perform counterfactual interventions on all diseases bar one ($D_k$, model shown in Figure [F]). We derive our expressions for three layer noisy-OR SCM's. Before proceeding, we motivate our choice of counterfactual query for the task of diagnosis.   \\

An observation will often have multiple possible causes, which constitute competing explanations. For example, the observation of a symptom $S = 1$ can in principle be explained by any of its parent diseases. In the case that a symptom has multiple associated causes (diseases), rarely is a single disease \emph{necessary} to explain a given symptom. Equivalently, the symptoms associated with a disease tend to be present in patient's suffering from this diseases, without \emph{requiring} a secondary disease to be present. This can be summarised by the following assumption -- \emph{any single disease is a sufficient cause of any of its associated symptoms}. Under this assumption, determining the likelihood that a diseases is causing a symptom reduces to simple deduction -- removing all other possible causes and seeing if the symptom remains. \\

The question of how we can define and quantify causal explanations in general models is an area of active research \cite{halpern2016actual,halpern2005causes,halpern2000axiomatizing,eiter2002complexity} and the approach we propose here cannot be applied to all conceivable SCMs. For example, if you had a symptom that can be present \emph{only if} two parents diseases $D_1$ and $D_2$ are both present, then neither of these parents in isolation is a sufficient cause (individually, $D_1=1$ and $D_2=1$ are necessary but not sufficient to cause $S=1$). In Appendix A.4 we present a different counterfactual query that captures causality in this case by reasoning about necessary treatments. However, in the case that our symptoms obey noisy-or statistics, all diseases are \emph{individually sufficient} to generate any symptom. This is ensured by the OR function, which states that a symptom $S$ is the Boolean OR of its parents individual activation functions, $s = \bigvee_{i=1}^N [d_i \wedge \bar u_{D_i, S}]$ where the activation function from parent $D_i$ is $f_i = d_i \wedge \bar u_{D_i, S}$. Thus, any single activation is sufficient to explain $S = 1$ and we can quantify the expected sufficiency of a diseases individually. An example of a model that would violate this property is a noisy-AND model, where $s = \bigwedge_{i=1}^N [d_i \wedge \bar u_{D_i, S}]$ - e.g. all parent diseases must be present in order for the symptom to be present.

Given these properties of noisy-OR models (as disease models in general), we propose our measure for quantifying how well a disease explains the patient's symptoms -- the \emph{expected sufficiency}. For a given disease, this measures the number of symptoms that we would expect to remain if we intervened to nullify all other possible causes of symptoms. This counterfactual intervention is represented by the causal model shown in figure [F] in appendix A.2.\\

\begin{repdefine}{def expected sufficiency}
The expected sufficiency of disease $D_k$ determines the number of positively evidenced symptoms we would expect to persist if we intervene to switch off all other possible causes of the symptoms, 

%\begin{equation}\label{expected disablement}
%    \mathbb E_\text{suff}(D_k, \mathcal E) := \sum\limits_{\mathcal S'}\!\left|\mathcal S_+' \right|P(\mathcal S' | \mathcal E, \text{do}(\mathcal D \setminus D_k = 0), \text{do}(\mathcal U_L = 0))
%\end{equation}

\begin{equation}\label{expected sufficiency}
    \mathbb E_\text{suff}(D_k, \mathcal E) := \sum\limits_{\mathcal S'}\!\left|\mathcal S_+' \right|P(\mathcal S' | \mathcal E, \text{do}(\mathsf{Pa}(\mathcal S_+) \setminus D_k = 0))
\end{equation}

\noindent where the expectation is calculated over all possible counterfactual symptom evidence states $\mathcal S'$ and $\mathcal S_+'$ denotes the positively evidenced symptoms in the counterfactual symptom evidence state. $\mathsf{Pa}(\mathcal S_+'\setminus D_k)$ denotes the set of all parents of the set of counterfactual positively evidenced symptoms $\mathcal S_+'$ excluding $D_k$, and $\text{do}(\mathsf{Pa}(\mathcal S_+) \setminus D_k = 0)$ denotes the counterfactual intervention setting $\mathsf{Pa}(\mathcal S_+'\setminus D_k)\rightarrow 0$. $\mathcal E$ denotes the set of all factual evidence. 

\end{repdefine}

To evaluate the expected sufficiency we must first determine the dual symptom CPTs in the corresponding twin network (figure [F]).

\begin{lemma}\label{simplified conditionals lemma 1}
For a given symptom $S$ and its counterfactual dual $S^*$, with parent diseases $\mathcal D$ and under the counterfactual interventions $\text{do}( \mathcal D \setminus D_k^*=0)$ and $\text{do}(\mathcal U^*_L = 0)$, the joint conditional distribution is given by 
\[
    P(s, s^* |\wedge_{i=1}^N d_i, \text{do}(\wedge_{i\neq k}D_i^* = 0), \text{do}(u_L^* = 0)) = \begin{cases}
    P(s = 0 | \wedge_{i=1}^N d_i), \quad s = s^* = 0\\
    0 , \quad s = 0, s^* = 1 \\
    \lambda_{D_k, s}^{d_k} P(s^{\setminus k} = 1 | \wedge_{i\neq k}d_i, D_k = 1), \quad s= 1, s^* = 0\\
    (1-\lambda_{D_k, S})\delta (d_k - 1) ,\quad  s = 1, s^* = 1
    \end{cases}
\]
where $\delta (d_k - 1) = 1$ if $D_k = 1$ else 0, and $\underline d$ is an instantiation of all $D_i\in \mathsf{Pa}(S)$, $\wedge_{i\neq k}D_i^*$ is the set of all counterfactual disease nodes excluding $D_k$, $\wedge_{i\neq k}d_i$ is the given instantiation on all disease nodes exlcuding $D_k$, and $u^*_L$ denotes the leak node for the counterfactual symptom. $s^{\setminus k}$ denotes the state of the factual symptom node $S$ under the graph surgery removing any direct edge from $D_k$ to $S$.

\end{lemma}

\begin{proof}
The CPT for the dual symptom nodes $S, S^*$ is given by
\begin{align*}\label{twin factorization}
    &P(s, s^* |\underline d, \text{do}(\wedge_{i\neq k}D_i^* = 0), \text{do}(u_L^* = 0)) =\\
    &\sum\limits_{u_{D_1, S}}P(u_{D_1, S})\cdots \sum\limits_{u_{D_N, S}}P(u_{D_N, S}) \sum\limits_{u_L}P(u_L)P(s | d_k, \wedge_{i\neq k}d_i, u_L)P(s^* | d_k,  \text{do}(\wedge_{i\neq k}D_i^* = 0), \text{do}(u_L^* = 0))\numberthis
\end{align*}

Where we have use the fact that the latent variables and the disease variables together form a Markov blanket for $S, S^*$, and we have used the conditional independence structure of the twin network, shown in Figure [F], which implies that $S$ and $S^*$ only share a single variable, $D_k$, in their Markov blankets. With the full Markov blanket specified, including the exogenous latents, the CPTs in \eqref{twin factorization} are deterministic functions, each taking the value 1 if their conditional constraints are satisfied. Note that the product of these two functions is equivalent to a function that is 1 if both sets of conditional constraints are satisfied and zero otherwise, and marginalizing over all latent variable states multiplied by this function is equivalent to the definition of the CPT for SCMs given in equation \eqref{fundamental scm equation}, where the CPT is determined by a conditional sum over the exogenous latent variables. Given the definition of the noisy-OR SCM in \eqref{eq: noisy-OR def}, these functions take the form 

\begin{equation}
    P(s | d_k, \wedge_{i\neq k}d_i, u_L) = \begin{cases}
    \bar u_L \bigwedge_{i=1}^N [\bar d_i \vee u_{D_i, S}], \quad s = 0 \\
    1-\bar u_L \bigwedge_{i=1}^N [\bar d_i \vee u_{D_i, S}], \quad s = 1
    \end{cases}
\end{equation}
and 
\begin{equation}
    P(s^* | d_k,  \text{do}(\wedge_{i\neq k}D_i^* = 0), \text{do}(u_L^* = 0)) = 
    \begin{cases}
    \bar d_k \vee u_{D_k, S}, \quad s^* = 0 \\
    1-\bar d_k \vee u_{D_k, S}, \quad s^* = 1
    \end{cases}
\end{equation}

Taking the product of these functions gives the function $g_{s, s^*}(\underline u, \underline d, u_L) := P(s | d_k, \wedge_{i\neq k}d_i, u_L)\times  P(s^* | d_k,  \text{do}(\wedge_{i\neq k}D_i^* = 0), \text{do}(u_L^* = 0))$ where $\underline u$ denotes a given instantiation of the free latent variables $u_{D_1, S}, \ldots, u_{D_N, S}$. 

\begin{equation}
    g_{s, s^*}(\underline u, \underline d, u_L) = \begin{cases}
    \bar u_L\bigwedge\limits_{i=1}^N[\bar d_i\vee u_{D_i, S}], \quad s = s^* = 0\\
    0 , \quad s=0, s^* = 1\\
    [\bar d_k \vee u_{D_k, S}]\wedge [1- \bigwedge\limits_{i=1}^N [\bar d_i \vee u_{D_i, S}]], \quad s = 1, s^* = 0\\
    1- \bar d_k \vee u_{D_k, S}, \quad s = 1, s^* = 1
    \end{cases}
\end{equation}

\begin{align}
        P(s, s^* |\underline d, \text{do}(\wedge_{i\neq k}D_i^* = 0), \text{do}(u_L^* = 0)) &= \sum\limits_{u_{D_1, S}}P(u_{D_1, S})\cdots \sum\limits_{u_{D_N, S}}P(u_{D_N, S}) \sum\limits_{u_L}P(u_L)g_{s, s^*}(\underline u, \underline d, u_L)\\
        &= \begin{cases}
    \lambda_L \prod\limits_{i=1}^N \lambda_{D_i, S}^{d_i}, \quad s = s^* = 0\\
    0 , \quad s = 0, s^* = 1 \\
    \lambda_{D_k, s}^{d_k} - \lambda_L \prod\limits_{i=1}^N \lambda_{D_i, S}^{d_i}, \quad s= 1, s^* = 0\\
    (1-\lambda_{D_k, S})\delta (d_k - 1) ,\quad  s = 1, s^* = 1
    \end{cases}
\end{align}

where we have used $\sum_{u_{D_i, S}}P(u_{D_i, S})\bar d_i \vee u_{D_i, S} = P(u_{D_i, S}=1) + P(u_{D_i, S} = 0) \bar d_i$ $= P(u_{D_i, S}=1)^{d_i} = \lambda_{D_i, S}^{d_i}$, and $\sum_{u_{D_k, S}}P(u_{D_k, S})[1- \bar d_k \vee u_{D_k, S}] = (1- \lambda_{D_k, S})\delta (d_k - 1) $, where $\delta (d_k -1)$ is 1 iff $D_k = 1$ and 0 otherwise. $\lambda_L \prod\limits_{i=1}^N \lambda_{D_i, S}^{d_i}$ can immediately be identified as $P(s = 0 | \mathcal D)$ by \eqref{noisy leaky or}. $\lambda_{D_k, s}^{d_k} - \lambda_L \prod\limits_{i=1}^N \lambda_{D_i, S}^{d_i} = \lambda_{D_k, s}^{d_k}(1 - \lambda_L\prod\limits_{i\neq k} \lambda_{D_i, S}^{d_i})$, and we can identify $\lambda_L\prod\limits_{i\neq k} \lambda_{D_i, S}^{d_i} = P(s = 0 | \wedge_{i\neq k}d_i, d_k = 0)$. Therefore $\lambda_{D_k, s}^{d_k} - \lambda_L \prod\limits_{i=1}^N \lambda_{D_i, S}^{d_i} = \lambda_{D_k, s}^{d_k}P(s = 1 | \wedge_{i\neq k}d_i, d_k = 0)$. Finally, we can express this as $\lambda_{D_k, s}^{d_k}P(s^{\setminus k} = 1 | \wedge_{i\neq k}d_i, D_k = 1)$, where $s^{\setminus k}$ is the instantiation of $S^{\setminus k}$ -- which is the variable generated by removing any directed edge $D_k \rightarrow S$ (or equivalently, replacing $\lambda_{D_k, S}$ with $1$).

\end{proof}

\noindent Given our expression for the symptom CPT on the twin network, we now derive the expression for the expected sufficiency.

\begin{reptheorem}{expected sufficiency theorem}
For noisy-OR networks described in Appendix A.1-A.4, the expected sufficiency of disease $D_k$ is given by 
\[
    \mathbb E_\text{suff}(D_k,\mathcal E) = \frac{1}{P(\mathcal S_\pm |\mathcal R)}\sum\limits_{\mathcal S\subseteq \mathcal S_+}|\mathcal S_+\setminus \mathcal S|P(\mathcal S_- = 0, \mathcal S^{\setminus k} = 1, D_k = 1| \mathcal R)\prod\limits_{S\in \mathcal S_+\setminus \mathcal S}(1-\lambda_{D_k, S})\prod\limits_{S\in \mathcal S}\lambda_{D_k, S}
\]

where $\mathcal S_\pm$ denotes the positive and negative symptom evidence, $\mathcal R$ denotes the risk-factor evidence, and $\mathcal S^{\setminus k}$ denotes the set of symptoms $\mathcal S$ with all directed arrows from $D_k$ to $S\in \mathcal S$ removed.  

\end{reptheorem}

\begin{proof}
Starting from the definition of the expected sufficiency 

\begin{equation}\label{exp cas suf}
        \mathbb E_\text{suff}(D_k, \mathcal E) := \sum\limits_{\mathcal S'}\!\left|\mathcal S_+' \right|P(\mathcal S' | \mathcal E, \text{do}(\mathcal D \setminus D_k = 0), \text{do}(\mathcal U_L = 0))
\end{equation}

we must find expressions for all CPTs $P(\mathcal S' |\mathcal E, \text{do}(\mathcal D \setminus D_k = 0), \text{do}(\mathcal U_L = 0))$ where $|\mathcal S_+'|\neq 0$ (terms with $\mathcal S_+' = \emptyset$ do not contribute to \eqref{exp cas suf}). Let $\mathcal S^*_A = \{S^* \text{ s.t. } S\in \mathcal S_-, S^*\in \mathcal S_-'\}$ (symptoms that remain off following the counterfactual intervention), $\mathcal S^*_B = \{S^* \text{ s.t. } S\in \mathcal S_+, S^*\in \mathcal S_+'\}$ (symptoms that remain on following the counterfactual intervention), and $\mathcal S^*_C = \{S^* \text{ s.t. } S\in \mathcal S_+, S^*\in \mathcal S_-'\}$ (symptoms that are switched off by the counterfactual intervention). Lemma \ref{simplified conditionals lemma 1} implies that $P(S = 0, S^*=1 |\underline d, \text{do}(\wedge_{i\neq k}D_i^* = 0), \text{do}(u_L^* = 0)) = 0$, and therefore these three cases are sufficient to characterise all possible counterfactual symptom states $\mathcal S'$. Therefore, to evaluate \eqref{exp cas suf}, we need only determine expressions for the following terms

\begin{equation}
    P(S_A^* = 0, S_B^* = 1, S_C^* = 0|\mathcal S_\pm, \mathcal R, \text{do}(\wedge_{i\neq k}D_i^* = 0), \text{do}(\mathcal U_L^* = 0) )
\end{equation}

where $\mathcal U_L^*$ denotes the set of all counterfactual leak nodes for the symptoms $\mathcal S_A^*, \mathcal S_B^*, \mathcal S_C^*$. Note that we only perform counterfactual interventions, i.e. interventions on counterfactual variables. As the exogenous latents are shared by the factual and counterfactual graphs, $\mathcal U_L^* = U_L$, but we maintain the notation for clarity. First, note that 

\begin{align*}
    &P(S_A^* = 0, S_B^* = 1, S_C^* = 0|\mathcal S_\pm, \mathcal R, \text{do}(\wedge_{i\neq k}D_i^* = 0), \text{do}(\mathcal U_L^* = 0) ) \\
    &=\frac{P(S_A^* = 0, S_B^* = 1, S_C^* = 0, \mathcal S_\pm | \mathcal R, \text{do}(\wedge_{i\neq k}D_i^* = 0), \text{do}(\mathcal U_L^* = 0) )}{P(\mathcal S_\pm | \mathcal R, \text{do}(\wedge_{i\neq k}D_i^* = 0), \text{do}(\mathcal U_L^* = 0) )}\\
    &= \frac{P(S_A^* = 0, S_B^* = 1, S_C^* = 0, \mathcal S_\pm | \mathcal R, \text{do}(\wedge_{i\neq k}D_i^* = 0), \text{do}(\mathcal U_L^* = 0) )}{P(\mathcal S_\pm | \mathcal R)}
\end{align*}

Which follows from the fact that the factual symptoms $\mathcal S_\pm$ on the twin network [F] are conditionally independent from the counterfactual interventions $\text{do}(\wedge_{i\neq k}D_i^* = 0), \text{do}(\mathcal U_L^* = 0) )$. To determine $Q = P(S_A^* = 0, S_B^* = 1, S_C^* = 0, \mathcal S_\pm | \mathcal R, \text{do}(\wedge_{i\neq k}D_i^* = 0), \text{do}(\mathcal U_L^* = 0) )$, we express $Q$ as a marginalization over the factual diseases which, together with the interventions on the counterfactual diseases and leak nodes, constitute a Markov blanket for each dual pair of symptoms

\begin{align*}
    Q &= \sum\limits_{d_1, \ldots,  d_N}P(\wedge_{i\neq k}D_i = d_i, D_k = d_k|\mathcal R)\prod\limits_{S\in \mathcal S_A}P(S^* = 0, S = 0 |\wedge_{i\neq k}D_i = d_i, D_k = d_k,\text{do}(\wedge_{i\neq k}D_i^* = 0), \text{do}(\mathcal U_L^* = 0))\\
    &\times \prod\limits_{S\in \mathcal S_B}P(S^* = 1, S = 1 |\wedge_{i\neq k}D_i = d_i, D_k = d_k,\text{do}(\wedge_{i\neq k}D_i^* = 0), \text{do}(\mathcal U_L^* = 0))\\
    &\times \prod\limits_{S\in \mathcal S_C}P(S^* = 0, S = 1 |\wedge_{i\neq k}D_i = d_i, D_k = d_k,\text{do}(\wedge_{i\neq k}D_i^* = 0), \text{do}(\mathcal U_L^* = 0))
    \numberthis
\end{align*}

Substituting in the CPT derived in Lemma \ref{simplified conditionals lemma 1} yields 

\begin{align*}
    Q =  \sum\limits_{d_1, \ldots,  d_N}P(\wedge_{i\neq k}D_i = d_i, D_k = d_k|\mathcal R)&\prod\limits_{S\in \mathcal S_A}P(s = 0 |\wedge_{i=1}^N d_i)\prod\limits_{S\in \mathcal S_B}(1-\lambda_{D_k, S})\delta (d_k - 1)\\
    &\times\prod\limits_{S\in \mathcal S_C}\lambda_{D_k, s}^{d_k} P(s^{\setminus k} = 1 | \wedge_{i\neq k}d_i, D_k = 1)\numberthis
\end{align*}

The only terms in \eqref{exp cas suf} with $|\mathcal S_+'|\neq 0$ have $\mathcal S_B\neq \emptyset$, therefore the term $\delta(d_k - 1)$ is present, and $Q$ simplifies to 

\begin{align*}
    Q =  \sum\limits_{d_i \forall i\neq k}P(\wedge_{i\neq k}D_i = d_i, D_k = 1|\mathcal R)&\prod\limits_{S\in \mathcal S_A}P(s = 0 |\wedge_{i\neq k}^N d_i, D_k = 1)\prod\limits_{S\in \mathcal S_B}(1-\lambda_{D_k, S})\\
    &\times \prod\limits_{S\in \mathcal S_C}\lambda_{D_k, s} P(s^{\setminus k} = 1 | \wedge_{i\neq k}d_i, D_k = 1)\numberthis
\end{align*}
\begin{equation}
    = P(S_A = 0, S^{\setminus k}_C = 1,D_k = 1|\mathcal R)\prod\limits_{S\in S_B}(1-\lambda_{D_k, S})\prod\limits_{S\in S_C}\lambda_{D_k, S}
\end{equation}

\noindent where in the last line we have performed the marginalization over $d_i$ $\forall$ $i\neq k$. Finally, $\mathcal S_+' = \mathcal S^*_B = \mathcal S_+\setminus \mathcal S_C$, and so $|\mathcal S_+'| = |\mathcal S_+|-|\mathcal S_C|$, and the expected expected sufficiency is 

\begin{equation}
    \mathbb E_\text{suff} (D_k,\mathcal E)= \frac{1}{{P(\mathcal S_\pm |\mathcal R)}}\sum\limits_{\mathcal S\subseteq \mathcal S_+}\left(|\mathcal S_+|-|\mathcal S| \right)P(\mathcal S_- = 0, \mathcal S^{\setminus k} = 1, D_k = 1| \mathcal R)\prod\limits_{S\in \mathcal S_+\setminus \mathcal S}(1-\lambda_{D_k, S})\prod\limits_{S\in \mathcal S}\lambda_{D_k, S}
\end{equation}

\noindent where we have dropped the subscript $C$ from $\mathcal S_C$.

\end{proof}

\noindent Given our expression for the expected sufficiency, we now derive a simplified expression that is very similar to the posterior $P(D_k = 1 |\mathcal R, \mathcal S_\pm )$.

\begin{theorem}[Simplified expected sufficiency]\label{theorem deriv exp suff}
\begin{equation}
    \mathbb E_\text{suff}(D_k,\mathcal E) = \frac{1}{P(\mathcal S_\pm |\mathcal R)}\sum\limits_{\mathcal Z\subseteq \mathcal S_+}(-1)^{|\mathcal Z|}P(\mathcal S_- = 0, \mathcal Z = 0, D_k = 1|\mathcal R) \times \tau (k, \mathcal  Z)
\end{equation}
where
\begin{equation}
    \tau (k, \mathcal  Z) =\sum\limits_{S\in \mathcal S_+\setminus \mathcal Z}(1-\lambda_{D_k, S})
\end{equation}

\end{theorem}

\begin{proof}
Starting with the expected sufficiency given in Theorem \ref{expected sufficiency theorem}, we can perform the change of variables $\mathcal X = \mathcal S_+\setminus \mathcal S$ to give

\begin{align}
    \mathbb E_\text{suff}(D_k,\mathcal E) &= \frac{1}{P(\mathcal S_\pm |\mathcal R)}\sum\limits_{\mathcal X\subseteq \mathcal S_+}|X|\prod\limits_{S\in \mathcal X}(1-\lambda_{D_k, S})\prod\limits_{S\in \mathcal S_+\setminus \mathcal X}\lambda_{D_k, S}\,P(\mathcal S_- = 0, (\mathcal S_+
    \setminus \mathcal X)^{\setminus k} = 1,D_k = 1| \mathcal R)\\
    &= \frac{1}{P(\mathcal S_\pm  |\mathcal R)}\sum\limits_{\mathcal X\subseteq \mathcal S_+}|\mathcal X|\prod\limits_{S\in \mathcal X}(1-\lambda_{D_k, S})\prod\limits_{S\in \mathcal S_+\setminus \mathcal X}\lambda_{D_k, S}\,\sum\limits_{\mathcal Z\subseteq \mathcal S_+\setminus \mathcal X}(-1)^{|\mathcal Z|}P(\mathcal S_- = 0, \mathcal Z^{\setminus k} = 0,D_k = 1| \mathcal R) 
\end{align}

where in the last line we apply the inclusion-exclusion principle to decompose an arbitrary joint state over Bernoulli variables $P(\mathcal A = 0, \mathcal B = 1)$ as a sum over the powerset of the variables $\mathcal B$ in terms of marginals where all variables are instantiated to 0,

\begin{equation}
    P(\mathcal A = 0, \mathcal B = 1) = \sum\limits_{\mathcal C \subseteq \mathcal B}(-1)^{|C|}P(\mathcal A = 0, \mathcal C = 0)
\end{equation}

By the definition of noisy-or \eqref{def eq noisy or} we have that 

\begin{align*}
    P(\mathcal S_- = 0, &\mathcal Z^{\setminus k} = 0,D_k = 1| \mathcal R)  \\
    &=\sum\limits_{d_i, i\neq k}P(\mathcal S_- = 0, \mathcal Z^{\setminus k} = 0,D_k = 1, \wedge_{i\neq k}^N D_i =  d_i| \mathcal R)\\
    &= \sum\limits_{d_i, i\neq k}\prod\limits_{S\in \mathcal S_-}P(S = 0| D_k = 1, \wedge_{i\neq k}^N D_i =  d_i)\prod\limits_{S\in \mathcal Z}P(S^{\setminus k}= 0|D_k = 1, \wedge_{i\neq k}^N D_i =  d_i) P(D_k = 1, \wedge_{i\neq k}^N D_i =  d_i| \mathcal R)\\
    &= \sum\limits_{d_i, i\neq k}\prod\limits_{S\in \mathcal S_-}P(S = 0| D_k = 1, \wedge_{i\neq k}^N D_i =  d_i)\prod\limits_{S\in \mathcal Z}\frac{P(S= 0|D_k = 1, \wedge_{i\neq k}^N D_i =  d_i)}{\lambda_{D_k, S}} P(D_k = 1, \wedge_{i\neq k}^N D_i =  d_i| \mathcal R)\\
    &= \frac{P(\mathcal S_- = 0, \mathcal Z = 0,D_k = 1| \mathcal R)}{\prod\limits_{S\in \mathcal Z}\lambda_{D_k, S}}\numberthis 
\end{align*}
Therefore we can replace the graph operation represented by $\setminus k$ by dividing the CPT by the product $\prod\limits_{S\in \mathcal Z}\lambda_{D_k, S}$. This allows $\mathbb E_\text{suff}$ to be expressed as 

\begin{equation}\label{first sum eq}
   \mathbb E_\text{suff}(D_k,\mathcal E) = \frac{1}{P(\mathcal S_\pm |\mathcal R)}\sum\limits_{\mathcal X\subseteq \mathcal S_+}|\mathcal X|\prod\limits_{S\in \mathcal X}(1-\lambda_{D_k, S})\prod\limits_{S\in \mathcal S_+\setminus \mathcal X}\lambda_{D_k, S}\,\sum\limits_{\mathcal Z\subseteq \mathcal S_+\setminus \mathcal X}(-1)^{|\mathcal Z|}P(\mathcal S_- = 0, \mathcal Z = 0,D_k = 1| \mathcal R) \frac{1}{\prod\limits_{S\in \mathcal Z}\lambda_{D_k, S}}
\end{equation}

We now aggregate the terms in the power sum that yield the same marginal on the symptoms (e.g. for fixed $\mathcal Z$). Every $\mathcal  X \in \mathcal  S_+\setminus \mathcal Z$ yields a single marginal $P(\mathcal S_- = 0, \mathcal Z = 0,D_k = 1| \mathcal R)$ and therefore if we express \eqref{first sum eq} as a sum in terms of $\mathcal Z$, where each term $P(\mathcal S_- = 0, \mathcal Z = 0,D_k = 1| \mathcal R)$ aggregates the a coefficient $K_{\mathcal Z}$ of the form $\mathbb E_\text{suff}(D_k,\mathcal E) = \sum_{\mathcal Z\subseteq \mathcal S_+} K_{\mathcal Z} P(\mathcal S_- = 0, \mathcal Z = 0,D_k = 1| \mathcal R)$ where

\begin{align*}
    K_{\mathcal Z} &= \frac{(-1)^{|\mathcal Z|}}{P(\mathcal S_\pm |\mathcal R)}\frac{1}{\prod\limits_{S\in \mathcal  Z}\lambda_{D_k, S}}\sum\limits_{\mathcal X\subseteq \mathcal S_+\setminus \mathcal Z}|\mathcal X|\prod\limits_{S\in \mathcal X}(1-\lambda_{D_k, S})\prod\limits_{S\in \mathcal S_+\setminus \mathcal X}\lambda_{D_k, S}\\
    &= \frac{(-1)^{|\mathcal Z|}}{P(\mathcal S_\pm |\mathcal R)}\frac{1}{\prod\limits_{S\in \mathcal  Z}\lambda_{D_k, S}}\sum\limits_{\mathcal X\subseteq \mathcal A}|\mathcal X|\prod\limits_{S\in \mathcal X}(1-\lambda_{D_k, S})\prod\limits_{S\in \mathcal A \setminus \mathcal X}\lambda_{D_k, S}\prod\limits_{S\in \mathcal Z}\lambda_{D_k, S}\\
    &= \frac{(-1)^{|\mathcal Z|}}{P(\mathcal S_\pm |\mathcal R)}\sum\limits_{\mathcal X\subseteq \mathcal A}|\mathcal X|\prod\limits_{S\in \mathcal X}(1-\lambda_{D_k, S})\prod\limits_{S\in \mathcal A \setminus \mathcal X}\lambda_{D_k, S}\numberthis \label{z coefficient}
\end{align*}
where $\mathcal A = \mathcal S_+\setminus \mathcal Z$. This can be further simplified using the identity

\begin{equation}
    \sum\limits_{A\subseteq B}|A|\prod\limits_{a\in A}(1-a)\prod\limits_{a'\in B\setminus A}a' = |B| - \sum\limits_{a\in B} a = \sum\limits_{a\in B}(1-a)
\end{equation}

which we now prove iteratively. First, consider the function $S(\mathcal B):=  \sum\limits_{\mathcal A\subseteq \mathcal B}\prod\limits_{a\in \mathcal A}(1-a)\prod\limits_{a'\in \mathcal B\setminus \mathcal A}a'$. Now, consider $S(\mathcal B + \{ c\})$. This function can be divided into two sums, one where $c\in \mathcal A$ and the other where $c\not\in \mathcal A$. Therefore 

\begin{equation}
    S(\mathcal B + \{c\}) =  \sum\limits_{\mathcal A\subseteq \mathcal B}\prod\limits_{a\in \mathcal A}(1-a)\prod\limits_{a'\in \mathcal B\setminus \mathcal A}a' c + \sum\limits_{\mathcal A\subseteq \mathcal B}\prod\limits_{a\in \mathcal A}(1-a)\prod\limits_{a'\in \mathcal B\setminus \mathcal A}a'(1-c) = S(\mathcal B)
\end{equation}

Starting with the empty set, $S(\emptyset) = 1$, it follows that $S(\mathcal B) = 1$ $\forall$ countable sets $\mathcal B$. Next, consider the function $G(\mathcal B):=  \sum\limits_{\mathcal A\subseteq \mathcal B}|\mathcal A|\prod\limits_{a\in \mathcal A}(1-a)\prod\limits_{a'\in \mathcal B\setminus \mathcal A}a'$, which is the form of the sum we wish to compute in \eqref{z coefficient}. Proceeding as before, we have 

\begin{align*}
    G(\mathcal B + \{ c\}) &= \sum\limits_{\mathcal A\subseteq B}|\mathcal A|\prod\limits_{a\in \mathcal A}(1-a)\prod\limits_{a'\in \mathcal B\setminus \mathcal A}a' c + \sum\limits_{\mathcal A\subseteq \mathcal B}(|\mathcal A|+1)\prod\limits_{a\in \mathcal A}(1-a)\prod\limits_{a'\in \mathcal B\setminus \mathcal A}a'(1-c)\\
    &= c G(\mathcal B) + (1-c) G(\mathcal B) + (1-c) S(\mathcal B)
\end{align*}

Using $S(\mathcal B) = 1$ we arive at the recursive formula $G(\mathcal B+\{c\}) = G(\mathcal B) + (1-c)$. Starting with $G(\emptyset) = 0$, and building the set $\mathcal B$ by recursively adding elements $c$ to the set, we arrive at the identity 

\begin{equation}\label{set expression}
    G(\mathcal B) = |\mathcal B| - \sum\limits_{a\in \mathcal B} a
\end{equation}

Using \eqref{set expression} we can simplify the coefficient \eqref{z coefficient}

\begin{equation}\label{final coeff form}
    \frac{(-1)^{|\mathcal Z|}}{P(\mathcal S_\pm |\mathcal R)}\sum\limits_{\mathcal X\subseteq \mathcal S_+\setminus \mathcal Z}|\mathcal X|\prod\limits_{S\in \mathcal X}(1-\lambda_{D_k, S})\prod\limits_{S\in (\mathcal S_+\setminus \mathcal Z) \setminus \mathcal X}\lambda_{D_k, S} =   \frac{(-1)^{|\mathcal Z|}}{P(\mathcal S_\pm |\mathcal R)}\sum\limits_{S\in S_+\setminus \mathcal Z}(1-\lambda_{D_k, S})
\end{equation}

Rearranging \eqref{first sum eq} as a summation over $\mathcal Z$ substituting in \eqref{final coeff form} gives

\begin{equation}
    \mathbb E_\text{suff}(D_k,\mathcal E) = \frac{1}{P(\mathcal S_\pm |\mathcal R)}\sum\limits_{\mathcal Z\subseteq \mathcal S_+}(-1)^{|\mathcal Z|}P(\mathcal S_- = 0, \mathcal Z = 0,D_k = 1| \mathcal R)\left( \sum\limits_{S\in \mathcal S_+\setminus \mathcal Z }(1-\lambda_{D_k, S})\right)
\end{equation}

which can be expressed as

\begin{equation}\label{final eq counterfactual QS}
    \mathbb E_\text{suff}(D_k,\mathcal E) = \frac{1}{P(\mathcal S_\pm |\mathcal R)}\sum\limits_{\mathcal Z\subseteq \mathcal S_+}(-1)^{|\mathcal Z|}P(\mathcal S_- = 0, \mathcal Z = 0, D_k = 1|\mathcal R) \times \tau (k, \mathcal  Z)
\end{equation}
where
\begin{equation}
    \tau (k, \mathcal  Z) =\sum\limits_{S\in \mathcal S_+\setminus \mathcal Z}(1-\lambda_{D_k, S})
\end{equation}

Note that if we fix $\tau (k, \mathcal  Z) = 1$ $\forall \mathcal  Z$, we recover $\sum\limits_{\mathcal Z\subseteq \mathcal S_+}(-1)^{|\mathcal Z|}P(\mathcal S_- = 0, \mathcal Z = 0, D_k = 1|\mathcal R)/ P(\mathcal S_\pm |\mathcal R) $ $= P(\mathcal S_\pm , D_k = 1|\mathcal R)/P(\mathcal S_\pm |\mathcal R) = P(D_k = 1 |\mathcal E)$, which is the standard posterior of disease $D_k$ under evidence $\mathcal E = \mathcal R \cap \mathcal S_\pm$ (this follows from the inclusion-exclusion principle, and can be easily checked by applying marginalization to express $P(\mathcal S_\pm , D_k = 1|\mathcal R)$ in terms of marginals where all symptoms are instantiated as 0). Note that \eqref{final eq counterfactual QS} can be seen as a counterfactual correction to the quickscore algorithm in \cite{heckerman1990tractable} (although we do not assume independence of diseases as the authors of \cite{heckerman1990tractable} do).

\end{proof}

\section{Properties of the expected sufficiency}\label{appendix: properties}

In this appendix, we show that the expected sufficiency \eqref{exp cas suf 2} obeys our four postulates, including an additional postulate of \emph{sufficiency} which is obeyed by the expected sufficiency.

\begin{theorem}[Diagnostic properties of expected sufficiency]

\begin{enumerate}
    \item{\emph{consistency}.}\quad $\mathbb E_\text{suff}(D_k, \mathcal E) \propto P(D_k = 1 |\mathcal E)$\\
    \item{\emph{causality}.}\quad If $\not \exists$ $S\in \mathsf{Dec}(D_k)\cap \mathcal S_+$ $\implies$ $\mathbb E_\text{suff}(D_k, \mathcal E) = 0$\\
    \item{\emph{simplicity}.}\quad  $\left|\mathbb  E_\text{suff}(D_k, \mathcal E)\right| \leq \left|\mathcal S_+\cap \mathsf{Dec}(D_k) \right|$\\
    \item{\emph{sufficiency}.}\quad  $\mathbb E_\text{suff}(D_i \wedge D_j, \mathcal E) >0$ $\implies$ $\mathbb E_\text{suff}(D_i, \mathcal E) >0$ and $\mathbb  E_\text{suff}(D_j, \mathcal E) >0$
\end{enumerate}

\end{theorem}
The expected sufficiency satisfies the following four properties,
\begin{proof}
Postulate 1 dictates that the measure should be proportional to the posterior probability of the diseases. Postulate 2 states that if the disease has no causal effect on the symptoms presented then it is a poor diagnosis and should be discarded. Postulate 3 states that the (tight) upper bound of the measure for a given disease (in the sense that there exists some disease model that achieves this upper bound -- namely deterministic models) is the number of positive symptoms that the disease can explain. This allows us to differentiate between diseases that are equally likely causes, but where one can explain more symptoms than another. Postulate 4 states that if it is possible that $D_k$ is causing at least one symptom, then the measure should be strictly greater than 0. Starting from the definition of the expected sufficiency

\begin{equation}\label{exp cas suf 2}
        \mathbb E_\text{suff}(D_k, \mathcal E) := \sum\limits_{\mathcal S'}\!\left|\mathcal S_+' \right|P(\mathcal S' | \mathcal E, \text{do}(\mathcal D \setminus D_k = 0), \text{do}(\mathcal U_L = 0))
\end{equation}

given the conditional independence structure of the twin network [F], we can express the counterfactual symptom marginals as 

\begin{align}
    P(\mathcal S' | \mathcal E, &\text{do}(\mathcal D \setminus D_k = 0), \text{do}(\mathcal U_L = 0)) \\
    &= \sum\limits_{d_k}\prod\limits_{S^*\in \mathcal S'}P(S^* | \mathcal E, \text{do}(\mathcal D^* \setminus D_k = 0), \text{do}(\mathcal U^*_L = 0), d_k)P(d_k | \mathcal E, \text{do}(\mathcal D^* \setminus D_k = 0), \text{do}(\mathcal U^*_L = 0))\\
    &= \sum\limits_{d_k}\prod\limits_{S^*\in \mathcal S'}P(S^* | \mathcal E, \text{do}(\mathcal D^* \setminus D_k = 0), \text{do}(\mathcal U^*_L = 0), d_k)P(d_k | \mathcal E)\
\end{align}

If $D_k = 1$, then do the the counterfactual interventions the counterfactual states have all parents (including leaks) instantiated to 0, which implies that $\mathcal S_+' = \emptyset$ by \eqref{def noisy scm}. Hence this case never contributes to the expected sufficiency as the expectation is over $|\mathcal S_+'|$. For $D_k = 1$, we recover that $ P(\mathcal S' | \mathcal E, \text{do}(\mathcal D \setminus D_k = 0), \text{do}(\mathcal U_L = 0))\propto P(D_k = 1 | \mathcal E)$ and therefore $\mathbb E_\text{suff}(D_k, \mathcal E) \propto P(D_k = 1 | \mathcal E)$. For postulate 2, if there are no symptoms that are descendants of $D_k$, then $\mathbb E_\text{suff}(D_k, \mathcal E) = 0$. This follows immediately from the fact that if $D_k$ is not an ancestor of any of the symptoms, then all counterfactual symptoms have all parents instantiated as $0$ and $\mathcal S_+' = \emptyset$. For postulate 4, we can only prove this property under additional assumptions about our disease model (see appendix \ref{appendix: scms} for noisy-and counter example). First, note that $\mathbb E_\text{suff}(D_k, \mathcal E)$ is a convex sum with positive semi-definite coefficients $|\mathcal S_+'|$. If there is a single positively evidenced symptom that is a descendent of $D_k$, and $D_k$ has a positive causal influence on that child, and our disease model permits that every disease be capable of causing its associated symptoms in isolation, i.e. $P(S = 1 |\text{only})(D_k = 1))> 0 $ for $S \in \mathsf{Dec}(D_k)$, then it is simple to check that $P(S^* =1| \mathcal E, \text{do}(\mathcal D^* \setminus D_k = 0), \text{do}(\mathcal U^*_L = 0), d_k=1)>0$ and so $E_\text{suff}(D_k, \mathcal E)>0$.
\end{proof}
\section{expected disablement}\label{appendix: expected dis}

In this appendix we turn our attention to our second diagnostic measure -- the expected disablement. This measure is closer to typical treatment measures, such as the efffect of treatment on the treated \cite{shpitser2009effects}. We use our twin diagnostic network outlined in appendix \ref{appendix: twin networks} figure [E] (shown below) to simulating counterfactual treatments. We focus on the simplest case of single disease interventions, and propose a simple ranking measure whereby the best treatments are those that get rid of the most symptoms.

\begin{repdefine}{def: expected disablement}
The expected disablement of disease $D_k$ determines the number of positive symptoms that we would expect to switch off if we intervened to turn off $D_k$,

\begin{equation}\label{def exp dis}
        \mathbb E_\text{dis}(D_k, \mathcal E) := \sum\limits_{ \mathcal S'}\!\left|\mathcal S_+ \setminus \mathcal S_+' \right|P(\mathcal S' | \mathcal E, \text{do}(D_k = 0))
\end{equation}

\noindent where $\mathcal E$ is the factual evidence and $\mathcal S_+$ is the set of factual positively evidenced symptoms. The expectation is calculated over all possible counterfactual symptom evidence states $\mathcal S'$ and $\mathcal S_+'$ denotes the positively evidenced symptoms in the counterfactual symptom evidence state. $\text{do}(D_k = 0)$ denotes the counterfactual intervention setting $D_k \rightarrow 0$.

\end{repdefine}

Decisions about which treatment to select for a patient generally take into account variables such as cost and cruelty. These variables can be simply included in the treatment measure. For example, the cruelty of specific symptoms can be included in the expectation \eqref{def exp dis} by weighting each positive symptom accordingly. The cost of treating a specific disease is included simply by multiplying \eqref{def exp dis} by a cost weight, and likewise for including the probability of the intervention succeeding. For now, we focus on computing the counterfactual probabilities, which we can then use to construct arbitrarily weighted expectations.

To calculate \eqref{def exp dis}, note that the only CPTs that differ from the original noisy-OR SCM are those for unmerged dual symptom nodes (i.e. children of the intervention node $D_k$). The disease layer forms a Markov blanket for the symptoms layer, d-separating dual symptom pairs from each other. Therefore we derive the CPT for dual symptoms and their parent diseases. \\

\[
\includegraphics[scale = 0.7]{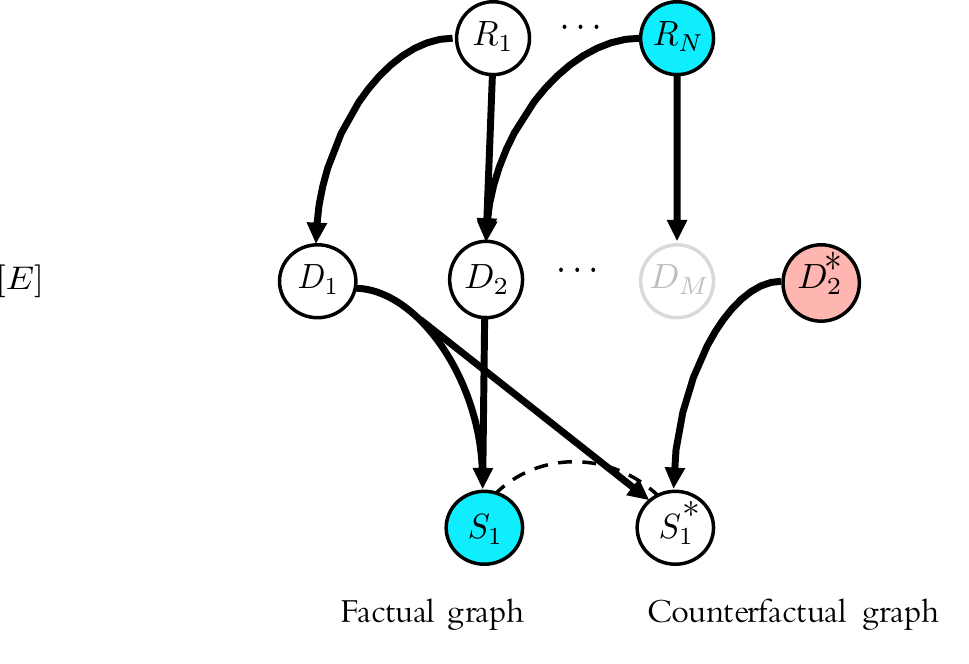}
\]

\begin{lemma}
\label{simplified conditionals}
For a given symptom $S$ and its counterfactual dual $S^*$, with parent diseases $\mathcal D$ and under the counterfactual intervention $\text{do}( D_k^*=0)$, the joint conditional distribution on the twin network is given by 
\[
    P(s, s^*\, |\,\wedge_{i}D_i = d_i, \text{do}( D_k^*=0))=\begin{cases}
   P(s = 0\,|\,  \wedge_{i}D_i = d_i) \quad \text{if } s = s^* = 0\\
   0 \quad \text{if } s = 0, s^* = 1\\
   \left(\frac{1}{\lambda_{D_k, S}}-1 \right)P(s= 0\,|\,  \wedge_{i\neq k}D_i = d_i, D_k = 1)\delta (d_k - 1)\quad \text{if } s = 1, s^* = 0 \, \text{ and } \lambda_{D_k, S}\neq 0\\
   P(s^{\setminus k}= 0\,|\,  \wedge_{i\neq k}D_i = d_i, D_k = 1)\delta (d_k - 1)\quad \text{if } s = 1, s^* = 0 \, \text{ and } \lambda_{D_k, S}= 0\\
    P(s^{\setminus k} =1 |  \wedge_{i\neq k}D_i = d_i, D_k=1)  \quad \text{if } s = 1, s^* = 1
   \end{cases}
\]
where $\delta (d_k - 1) = 1$ if $D_k = 1$ else 0. 

\end{lemma}

\begin{proof}
First note that for this marginal distribution the intervention $\text{do}( D_k^*=0)$ is equivalent to setting the evidence $D_k ^*= 0$ as we specify the full Markov blanket of $(s, s^*)$. Let $\mathcal D_{\setminus k}$ denote the set of parents of $(s, s^*)$ not including the intervention node $D_k^*$ or its dual $D_k$. We wish to compute the conditional probability
\begin{equation}\label{bool 1}
    P(s, s^*\, |\,   \wedge_{i\neq k}D_i = d_i, D_k = d_k) = \sum\limits_{\underline u_s} p (\underline u_s)P(s| \wedge_{i\neq k}D_i = d_i, D_k = d_k, u_s)P(s^*| \mathcal \wedge_{i\neq k}D_i = d_i, D^*_k=0, u_s)
\end{equation}

where $p (\underline u_s)$ is the product distribution over all exogenous noise terms for $S$ including the leak term. We proceed as before by expressing this as a marginalization over the CPT of the dual states,  $P(s=0, s^* = 0\, |\,  \wedge_{i\neq k}D_i = d_i, D_k^*=0, D_k )$, $P(s = 0\, |\,  \wedge_{i\neq k}D_i = d_i, D_k^*=0, D_k=d_k)$ and $P(s^*=0\, |\,  \mathcal \wedge_{i\neq k}D_i = d_i,D_k^*=0, D_k=d_k)$. For $s_i = 0$, the generative functions are given by 

\begin{equation}
    P(s=0 \, |\,  \mathsf{Pa}(S), u_s) = u_L \bigwedge\limits_{D_i\in \mathsf{Pa}(S)}(\bar d_i \vee u_{D_i, S})
\end{equation}

First we compute the joint state.

\begin{align*}
    &P(s =0|  \wedge_{i\neq k}D_i = d_i, D_k=d_k, u_s) P(s^* =0| \mathcal  \wedge_{i\neq k}D_i = d_i, D_k^* = d_k^*, u_s) \\
    &=u_L \wedge u_L\bigwedge\limits_{D_i\in \mathcal D_{\setminus k}}\left( u_{D_i, S} \vee \bar d_i\right)\bigwedge\limits_{D_j\in \mathcal D_{\setminus k}}\left( u_{D_j, S} \vee \bar d_j\right)\wedge \left[u_{D_k, S}\vee \bar d_k\right]\wedge \left[u_{D_k, S}\vee \bar d^*_k\right]\\
    &= u_L \bigwedge\limits_{D_i\in \mathcal D_{\setminus k}}\left( u_{D_i, S} \vee \bar d_i\right) \left[u_{D_k, S}\vee \left(\bar d^*_k\wedge \bar d_k\right)\right]
\end{align*}

Where we have used the Boolean identities $a \wedge a = a$ and $a \vee (b\wedge c) = (a \vee b )\wedge (a \vee c)$. Therefore

\begin{align*}
    P(s=0 , s^* = 0| \wedge_{i\neq k}D_i = d_i, D_k=d_k, D_k^*=d_k^*) &=  \sum\limits_{\underline u_s}p (\underline u_s) P(s=0 | \mathcal D_{\setminus k}, D_k, u_s) P(s^* =0| \mathcal D_{\setminus k}, D_k^*, u_s) \\
    &= \lambda_{L_s}\left[\lambda_{D_k, S}\left(d_k\vee d_k^*\right) +  \bar d_k\wedge \bar d_k^*\right]\prod\limits_{D_i\in \mathcal D_{\setminus k}}\left[\lambda_{D_i, S}d_i+  \bar d_i\right]
\end{align*}

Next, we calculate the single-symptom conditionals

\begin{align*}
    P(s =0\, | \, \wedge_{i\neq k}D_i = d_i, D_k=d_k ) &= \sum\limits_{\underline u_s} p (\underline u_s)P(s=0| \mathcal D_{\setminus k}, D_k, u_s) \\
    &= 
    \sum\limits_{u_{L_S}}P(u_{L_s})u_{L_s}\prod\limits_{D_i\in \mathcal D}\sum\limits_{u_{D_i, S}}P(u_{D_i, S})u_{D_i, S}\vee \bar d_i\\
    &= 
    P(u_{L_s}=1)\prod\limits_{D_i\in \mathcal D}\sum\limits_{u_{D_i, S}}\left[P(u_{D_i, S}=1)+ P(u_{D_i, S}=0) \bar d_i \right]\\
    &=  
    \lambda_{L_s}\prod\limits_{D_i\in \mathcal D}\left[\lambda_{D_i, S}d_i+  \bar d_i \right] \numberthis\label{single_cpd}
\end{align*}
and similar for $ P(s^* =0\, | \,\mathcal  \wedge_{i\neq k}D_i = d_i, D_k^* = d_k^*)$. Note that $\lambda x + \bar x = \lambda^{x}$. We can now express the joint cpd over dual symptom pairs, using the identities $P(s= 0, s^* = 1\, | \, X) = P(s = 0\, | \, X) - P(s=0, s^* = 0 \, | \, X)$, $P(s= 1, s^* = 0\, | \, X) = P(s^* = 0\, | \, X) - P(s=0, s^* = 0 \, | \, X)$ and $P(s= 1, s^* = 1\, | \, X) =1- P(s = 0\, | \, X) - P(s^* = 0\, | \, X)+ P(s=0, s^* = 0 \, | \, X)$ for arbitrary conditional $X$.

\begin{align*}
   P(s, s^*|  \wedge_{i\neq k}D_i = d_i, D_k=d_k, D_k^*=d_k^*) 
   &= \begin{cases}
   \lambda_{L_s}\lambda_{D_k, S}^{d_k\vee d_k^*}\prod\limits_{D_i\in \mathcal D_{\setminus k}}\lambda_{D_i, S}^{d_i} \quad \text{ if } s = s^* = 0\\
   \lambda_{L_s}\left[\lambda_{D_k, S}^{d_k}-\lambda_{D_k, S}^{d_k\vee d_k^*}\right]\prod\limits_{D_i\in \mathcal D_{\setminus k}}\lambda_{D_i, S}^{d_i} \quad \text{ if } s = 0, s^* = 1\\
   \lambda_{L_s}\left[\lambda_{D_k, S}^{d^*_k}-\lambda_{D_k, S}^{d_k\vee d_k^*}\right]\prod\limits_{D_i\in \mathcal D_{\setminus k}}\lambda_{D_i, S}^{d_i} \quad \text{ if } s = 1,  s^* = 0\\
    1- \lambda_{L_s}\left[\lambda_{D_k, S}^{d_k} + \lambda_{D_k, S}^{d^*_k} -\lambda_{D_k, S}^{d_k\vee d_k^*} \right]\prod\limits_{D_i\in \mathcal D_{\setminus k}}\lambda_{D_i, S}^{d_i} \quad \text{ if } s = s^* = 1
   \end{cases}
\end{align*}

As we are always intervening to switch off diseases, $D_k^* = 0$, then $d_k\vee d_k^* = d_k$ and 
\begin{equation}
    \lambda_{D_k, S}^{d_k}-\lambda_{D_k, S}^{d_k\vee d_k^*} = 0
\end{equation}

and therefore $P(s = 0, s^*=1 | \mathcal  \wedge_{i\neq k}D_i = d_i, D_k = d_k, D_k^*=0) = 0$ as expected (switching off a disease will never switch on a symptom). This simplifies our expression for the conditional distribution to 

\begin{equation}
    P(s, s^*|  \wedge_{i\neq k}D_i = d_i, D_k=d_k, D_k^*=0) = \begin{cases}
   \lambda_{L_s}\lambda_{D_k, S}^{d_k}\prod\limits_{D_i\in \mathcal D_{\setminus k}}\lambda_{D_i, S}^{d_i} \quad \text{ if } s = s^* = 0\\
   0\quad \text{ if } s = 0, s^* = 1\\
   \lambda_{L_s}\left[1-\lambda_{D_k, S}^{d_k}\right]\prod\limits_{D_i\in \mathcal D_{\setminus k}}\lambda_{D_i, S}^{d_i}\quad \text{ if } s = 1,  s^* = 0\\
    1- \lambda_{L_s}\prod\limits_{D_i\in \mathcal D_{\setminus k}}\lambda_{D_i, S}^{d_i}\quad \text{ if } s = s^* =1 
   \end{cases}
\end{equation}

This then simplifies using \eqref{single_cpd} to 

\begin{equation}\label{simplified conditionals}
    P(s, s^*|  \wedge_{i\neq k}D_i = d_i, D_k=d_k, D_k^*=0)=\begin{cases}
   P(s=0| \wedge_{i}D_i = d_i)\quad \text{ if } s = s^* = 0\\
   0\quad \text{ if } s = 0, s^* = 1\\
   P(s = 0|  \wedge_{i\neq k}D_i = d_i, D_k=0)-P(s=0|  \wedge_{i\neq k}D_i = d_i, D_k = d_k)\text{ if } s = 1,  s^* = 0\\
    P(s = 1|  \wedge_{i\neq k}D_i = d_i, D_k=0)  \quad \text{ if } s = s^* =1
   \end{cases}
\end{equation}

We have arrived at expressions for the CPT's over dual symptoms in terms of CPT's on the factual graph, and hence our conterfactual query can be computed on the factual graph alone. The third term in \eqref{simplified conditionals}, $ P(s = 0|  \wedge_{i\neq k}D_i = d_i, D_k=0)-P(s=0|  \wedge_{i\neq k}D_i = d_, D_k = d_k)$, equals zero unless $D_k = 1$. Using the definition of noisy-OR \eqref{def eq noisy or} to give

\begin{equation}
P(s = 0| \wedge_{i\neq k}D_i = d_i, D_k = 0) = \frac{1}{\lambda_{D_k, S}}P(s=0| \wedge_{i\neq k}D_i = d_i, D_k = 1)
\end{equation}

in the case that $\lambda_{D_k, S} > 0$, we recover

\begin{equation}
    P(s = 0|  \wedge_{i\neq k}D_i = d_i, D_k=0)-P(s=0|  \wedge_{i\neq k}D_i = d_i, D_k= d_k) =  \left(\frac{1}{\lambda_{D_k, S}}-1 \right)P(s= 0\,|\,  \wedge_{i\neq k}D_i = d_i, D_k = 1)\delta (d_k - 1)
\end{equation}

where $d_k$ is the instantiation of $D_k$ on the factual graph. The term $\delta (d_k - 1)$ is equivalent to fixing the observation $D_k = 1$ on the factual graph. If $\lambda_{D_k, S} = 0$ then 

\begin{equation}
     \lambda_{L_s}\left[1-\lambda_{D_k, S}^{d_k}\right]\prod\limits_{D_i\in \mathcal D_{\setminus k}}\lambda_{D_i, S}^{d_i} =  \lambda_{L_s}\prod\limits_{D_i\in \mathcal D_{\setminus k}}\lambda_{D_i, S}^{d_i}\delta(d_k-1)
\end{equation}

which is equivalent to $P(s^{\setminus k} = 0|\wedge_{i\neq k}D_i = d_i, D_k=1 )\delta(d_k-1)$

Finally, from the definition of the noisy-OR CPT \eqref{def noisy scm},

\begin{equation}
    P(s = 1|  \wedge_{i\neq k}D_i = d_i, D_k=0) = P(s^{\setminus k} =1 |  \wedge_{i\neq k}D_i = d_i, D_k=1)
\end{equation}

\end{proof}

Lemma \ref{simplified conditionals} allows us to express the expected disablement in terms of factual probabilities. As we have seen, the intervention $ \text{do}(D^*_k = 0)$ can never result in counterfactual symptoms that are on, when their dual factual symptoms are off, so we need only enumerate over counterfactual symptoms states where $\mathcal S_+'\subseteq \mathcal S_+$ as these are the only counterfactual states with non-zero weight. From this it also follows that for all $s\in \mathcal S_-$ $\implies$ $s^* \in \mathcal S_-'$. The counterfactual CPT in  \eqref{def exp dis} is represented on the twin network [F] as

\begin{equation}\label{simplified marginal 2}
    P(\mathcal S_+', \mathcal S_-'  | \mathcal E, \text{do}(D_k^* = 0)) =  P(\mathcal S_+',\mathcal  S_-'| \mathcal S_+,\mathcal  S_-, \mathcal R,  \text{do}(D^*_k = 0)) 
\end{equation}

\begin{theorem}[Simplified noisy-OR expected disablement]\label{theorem deriv expected dis}

For the noisy-OR networks described in Appendix \ref{appendix: scms}, the expected disablement of disease $D_k$ is given by

\begin{equation}
        \mathbb E_\text{dis}(D_k,\mathcal E) =
    \frac{1}{P(\mathcal S_+,\mathcal  S_-| \mathcal R)} \sum\limits_{\mathcal Z\subseteq \mathcal S_+}(-1)^{|\mathcal Z|}P(\mathcal S_- = 0, \mathcal Z = 0, D_k = 1|\mathcal R)\gamma(\mathcal Z, D_k)
\end{equation}

where

\begin{equation}
    \gamma (\mathcal Z, D_k) = \sum\limits_{S\in \mathcal Z}\left(1 - \frac{1}{\lambda_{D_k, S}}\right)
\end{equation}

where $\mathcal S_\pm$ is the set of factual positive (negative) evidenced symptom nodes and $\mathcal R$ is the risk factor evidence.

\end{theorem}

\begin{proof}
From the above discussion, the non-zero contributions to the expected disablement are 

\begin{equation}
    \mathbb E(D_k,\mathcal E)_\text{dis} = \sum\limits_{\mathcal C\subseteq \mathcal S_+}|\mathcal C|P(\mathcal S_-^* = 0,\mathcal C^* = 0,\mathcal S_+\setminus \mathcal C = 1 | \mathcal S_+,\mathcal  S_-, \mathcal R,  \text{do}(D^*_k = 0))
\end{equation}
Applying Bayes rule, and noting the the factual evidence states are not children of the intervention node $D_k^*$, gives 

\begin{equation}\label{lil eq}
    \mathbb E(D_k,\mathcal E)_\text{dis} =\frac{1}{P(\mathcal S_+,\mathcal  S_-| \mathcal R)} \sum\limits_{\mathcal C\subseteq \mathcal S_+}|\mathcal C|P(\mathcal S_-^* = 0,\mathcal C^* = 0,\mathcal S_+\setminus \mathcal C = 1 , \mathcal S_+,\mathcal  S_-| \mathcal R,  \text{do}(D^*_k = 0))
\end{equation}

\begin{comment}
Furthermore, by applying the inclusion-exclusion principle \cite{rota1964foundations}, we can expand this sum over the subsets $\mathcal X\subseteq \mathcal S_+\setminus \mathcal C$ to give an expression the right hand side of \eqref{lil eq} where all counterfactual variables are in the zero state

\begin{equation}\label{lil eq 2}
    \mathbb E(D_k,\mathcal E) =\frac{1}{P(\mathcal S_+,\mathcal  S_-| \mathcal R)} \sum\limits_{\mathcal C\subseteq \mathcal S_+}|\mathcal C|\sum\limits_{\mathcal X\subseteq \mathcal S_+\setminus \mathcal C}(-1)^{|\mathcal S_+\setminus \mathcal C|-|\mathcal X|}P(\mathcal S_-^* = 0,\mathcal C^* = 0,\mathcal X = 0 , \mathcal S_+,\mathcal  S_-| \mathcal R,  \text{do}(D^*_k = 0))
\end{equation}
\end{comment}
Let us now consider the probabilities $Q = P(\mathcal S_-^* = 0,\mathcal C^* = 0,\mathcal S\setminus \mathcal C^* = 1 , \mathcal S_+,\mathcal  S_-| \mathcal R,  \text{do}(D^*_k = 0))$. We can express these as marginalizations over the disease layer, which d-separate dual symptom pairs from each-other. First, we express $Q$ in the instance where we assume all $\lambda_{D_k, S}>0$.

\begin{align*}
    Q &= \sum\limits_{d, d_k}P(\wedge_{i\neq k}D_i = d_i, D_k = d_k|\mathcal R)\prod\limits_{S\in \mathcal S_-}P(S^* = 0, S = 0 |\wedge_{i\neq k}D_i = d_i, D_k = d_k, D_k^* = 0)\\
    &\times \prod\limits_{S\in \mathcal C}P(S^* = 0, S = 1 |\wedge_{i\neq k}D_i = d_i, D_k = d_k, D_k^* = 0)\prod\limits_{S\in \mathcal S_+\setminus \mathcal C}P(S^* = 1, S = 1 |\wedge_{i\neq k}D_i = d_i, D_k = d_k, D_k^* = 0)\numberthis\label{q equation}
\end{align*}

$\mathbb E(D_k,\mathcal E)$ is a sum of products of $Q$'s, therefore if all $Q$ are continuous for $\lambda_{D_k, S}\rightarrow 0$ $\forall$ $S$ we can derive $\mathbb E(D_k,\mathcal E)$ for positive $\lambda_{D_k, S}$ and take the limit $\lambda_{D_k, S}\rightarrow 0$ where appropriate. We can consider each term in isolation, as the product of continuous functions is continuous. Each term in $Q$ derives from one of

\begin{align*}
    P(s, s^*\, |\,&\wedge_{i\neq k}D_i = d_i, D_k=d_k,  \text{do}( D_k^*=0))\\
    &=\begin{cases}
   P(s = 0\,|\,  \wedge_{i}D_i = d_i) \quad \text{if } s = s^* = 0\\
   0 \quad \text{if } s = 0, s^* = 1\\
   \left(\frac{1}{\lambda_{D_k, S}}-1 \right)P(s= 0\,|\,  \wedge_{i\neq k}D_i = d_i, D_k = 1)\delta (d_k - 1)\quad \text{if } s = 1, s^* = 0 \, \text{ and } \lambda_{D_k, S}\neq 0\\
   P(s^{\setminus k}= 0\,|\,  \wedge_{i\neq k}D_i = d_i, D_k = 1)\delta (d_k - 1)\quad \text{if } s = 1, s^* = 0 \, \text{ and } \lambda_{D_k, S}= 0\\
    P(s^{\setminus k} =1 |  \wedge_{i\neq k}D_i = d_i, D_k=1)  \quad \text{if } s = 1, s^* = 1
   \end{cases}\numberthis
\end{align*}

Starting with $P(s = 0\,|\,  \wedge_{i}D_i = d_i) = \lambda_{L_S}\prod_{i=1}^N \lambda_{D_i, S}^{d_i}$, this is a linear function of $\lambda_{D_k, S}$ and therefore continuous in the limit $\lambda_{D_k, S}\rightarrow 0$. Secondly, 

\begin{equation}
     \left(\frac{1}{\lambda_{D_k, S}}-1 \right)P(s= 0\,|\,  \wedge_{i\neq k}D_i = d_i, D_k = 1)\delta (d_k - 1) = \left(\frac{1}{\lambda_{D_k, S}}-1 \right)\lambda_{L_S}\prod_{i=1}^N \lambda_{D_i, S}^{d_i}\delta (d_k - 1)
\end{equation}

which again is a linear function fo $\lambda_{D_k, S}$ and so is continuous in the limit $\lambda_{D_k, S}\rightarrow 0$. $P(s^{\setminus k}= 0\,|\,  \wedge_{i\neq k}D_i = d_i, D_k = 1)\delta (d_k - 1)$ is a constant function w.r.t $\lambda_{D_k, S}$, as is $ P(s^{\setminus k} =1 |  \wedge_{i\neq k}D_i = d_i, D_k=1)$, so these are also both continuous in the limit.

We therefore proceed under the assumption that $\lambda_{D_k, S} > 0$ $\forall$ $S$. Applying Lemma 1 simplifies \eqref{q equation} to 

\begin{align*}
    Q &= \sum\limits_{d}P(\wedge_{i\neq k}D_i = d_i, D_k = d_k|\mathcal R)\prod\limits_{S\in \mathcal S_- }P(S = 0|\wedge_{i\neq k}D_i = d_i, D_k = d_k)\prod\limits_{S\in  \mathcal C}P(S = 0|\wedge_{i\neq k}D_i = d_i, D_k = 1)\delta (d_k -1) \\
    &\times \prod\limits_{S\in  \mathcal S_+\setminus \mathcal C }P(S^{\setminus k} = 1| \wedge_{i\neq k}D_i = d_i, D_k = 1) \prod\limits_{S\in\mathcal C }\left(\frac{1}{\lambda_{D_k, S}}-1 \right) \numberthis
\end{align*}

\noindent Note that the only $Q$ that are not multiplied by a factor $|\mathcal C|=0$ in \eqref{lil eq} have $\mathcal C \neq \emptyset$, and so $\delta (d_k -1)$ is always present. Marginalizing over all disease states gives 

\begin{equation}\label{simple q}
    Q = P(\mathcal S_- = 0,  \mathcal C = 0, (\mathcal S_+\setminus \mathcal C)^{\setminus k} = 1, D_k = 1 | \mathcal R)\prod\limits_{S\in\mathcal C }\left(\frac{1}{\lambda_{D_k, S}}-1 \right)
\end{equation}

As before, we simplify this using a change of varaibles and the inclusion-exclusion principle. Change variables $\mathcal C \rightarrow \mathcal S_+\setminus \mathcal C$, which along with \eqref{simple q} gives

\begin{equation}
    \mathbb E(D_k,\mathcal E)_\text{dis} =\frac{1}{P(\mathcal S_+,\mathcal  S_-| \mathcal R)} \sum\limits_{\mathcal C\subseteq \mathcal S_+}|\mathcal S_+\setminus \mathcal C|P(\mathcal S_- = 0,(\mathcal S_+\setminus \mathcal C)= 0, \mathcal C^{\setminus k} = 1 ,D_k = 1| \mathcal R)\prod\limits_{S\in (\mathcal S_+\setminus \mathcal C)}\left(\frac{1}{\lambda_{D_k, S}}-1 \right)
\end{equation}

Next we apply the inclusion exclusion principle, giving

\begin{equation}
    \mathbb E(D_k,\mathcal E)_\text{dis} =\frac{1}{P(\mathcal S_+,\mathcal  S_-| \mathcal R)} \sum\limits_{\mathcal C\subseteq \mathcal S_+}|\mathcal S_+\setminus \mathcal C|\prod\limits_{S\in (\mathcal S_+\setminus \mathcal C)}\left(\frac{1}{\lambda_{D_k, S}}-1 \right)\sum\limits_{\mathcal Z\subseteq \mathcal C}(-1)^{|\mathcal Z|}P(\mathcal S_- = 0,(\mathcal S_+\setminus \mathcal C) = 0,\mathcal Z^{\setminus k} = 0 ,D_k = 1| \mathcal R)
\end{equation}

We can now proceed as before and remove the graph cut operation on the set $\mathcal Z$, using the definition of noisy-or \eqref{def noisy scm},

\begin{align*}
    P(\mathcal S_- = 0, &(\mathcal S_+\setminus \mathcal C) = 0, \mathcal Z^{\setminus k} = 0,D_k = 1| \mathcal R) \\
    &= \sum\limits_{d_i, i\neq k}P(\mathcal S_- = 0, (\mathcal S_+\setminus \mathcal C) = 0, \mathcal Z^{\setminus k} = 0,D_k = 1, \wedge_{i\neq k}^N D_i =  d_i| \mathcal R)\\
    &= \sum\limits_{d_i, i\neq k}\prod\limits_{S\in \mathcal S_\pm \setminus \mathcal C}P(S = 0| D_k = 1, \wedge_{i\neq k}^N D_i =  d_i)\prod\limits_{S\in \mathcal Z}P(S^{\setminus k}= 0|D_k = 1, \wedge_{i\neq k}^N D_i =  d_i) P(D_k = 1, \wedge_{i\neq k}^N D_i =  d_i| \mathcal R)\\
    &= \sum\limits_{d_i, i\neq k}\prod\limits_{S\in \mathcal S_\pm \setminus \mathcal C}P(S = 0| D_k = 1, \wedge_{i\neq k}^N D_i =  d_i)\prod\limits_{S\in \mathcal Z}\frac{P(S= 0|D_k = 1, \wedge_{i\neq k}^N D_i =  d_i)}{\lambda_{D_k, S}} P(D_k = 1, \wedge_{i\neq k}^N D_i =  d_i| \mathcal R)\\
    &= \frac{P(\mathcal S_- = 0,(\mathcal S_+\setminus \mathcal C) = 0, \mathcal Z = 0,D_k = 1| \mathcal R)}{\prod\limits_{S\in \mathcal Z}\lambda_{D_k, S}}\numberthis 
\end{align*}

Therefore

\begin{align*}
    \mathbb E(D_k,\mathcal E)_\text{dis} &=\frac{1}{P(\mathcal S_+,\mathcal  S_-| \mathcal R)} \sum\limits_{\mathcal C\subseteq \mathcal S_+}|\mathcal S_+\setminus \mathcal C|\prod\limits_{S\in \mathcal S_+\setminus \mathcal C}\left(\frac{1}{\lambda_{D_k, S}}-1 \right)\\
    &\times \sum\limits_{\mathcal Z\subseteq \mathcal C}(-1)^{|\mathcal Z|}P(\mathcal S_- = 0,\mathcal S_+\setminus \mathcal C = 0,\mathcal Z = 0,D_k = 1| \mathcal R)\prod\limits_{S\in \mathcal Z}\frac{1}{\lambda_{D_k, S}}
\end{align*}

%Therefore 

%\begin{equation}
%        \mathbb E(D_k,\mathcal E)_\text{dis} =\frac{1}{P(\mathcal S_+,\mathcal  S_-| \mathcal R)} \sum\limits_{\mathcal C\subseteq \mathcal S_+}|\mathcal S_+\setminus \mathcal C|\prod\limits_{S\in \mathcal S_+\setminus \mathcal C}\left(\frac{1}{\lambda_{D_k, S}}-1 \right) \sum\limits_{\mathcal Z\subseteq \mathcal C}(-1)^{|\mathcal Z|}P(\mathcal S_- = 0,\mathcal S_+\setminus \mathcal C = 0,\mathcal Z= 0 ,D_k = 1| \mathcal R)\prod\limits_{S\in \mathcal Z}\frac{1}{\lambda_{D_k, S}}
%\end{equation}

Finally, we aggregate all terms that have the same symptom marginal. Perform the change of variables $\mathcal X = \mathcal S_+\setminus \mathcal C$\\

\begin{equation}
    \mathbb E(D_k,\mathcal E)_\text{dis} =\frac{1}{P(\mathcal S_+,\mathcal  S_-| \mathcal R)} \sum\limits_{\mathcal X\subseteq \mathcal S_+}|\mathcal X|\prod\limits_{S\in \mathcal X}\left(\frac{1}{\lambda_{D_k, S}}-1 \right)\sum\limits_{\mathcal Z\subseteq \mathcal S_+\setminus \mathcal X}(-1)^{|\mathcal Z|}P(\mathcal S_- = 0,\mathcal X = 0,\mathcal Z = 0 ,D_k = 1| \mathcal R)\prod\limits_{S\in \mathcal Z}\frac{1}{\lambda_{D_k, S}}
\end{equation}

Clearly each term for a given $\mathcal X$ is zero unless $\lambda_{D_k, S}<1$ $\forall$ $S\in \mathcal X$, and so we can restrict ourselves to $\mathcal S \subseteq \mathcal S_+\cap \mathsf{Ch}(D_k)$. Furthermore, if any $\lambda_{D_k, S} = 0$ for $S \in \mathcal X$, then the symptom marginal (which is linearly dependent on $\lambda_{D_k, S}$) is 0 (there is zero probability of observing this symptom to be off if $D_k = 1$), and this term in the sum is zero. Therefore we can restrict the sum to $\mathcal X \subseteq S_+^{(k)}(\lambda > 0)$, where $S_+^{(k)}(\lambda > 0)$ is the set of positively evidenced factual symptoms that are children of $D_k$ and have $\lambda_{D_k, S}> 0$.  Let $\mathcal A = \mathcal X\cup \mathcal Z$. Each marginal $P(\mathcal S_- = 0,\mathcal A = 0 ,D_k = 1| \mathcal R)$ aggregates a coefficient

\begin{equation}
    \frac{1}{P(\mathcal S_+,\mathcal  S_-| \mathcal R)} \sum\limits_{\mathcal X\subseteq \mathcal A}|\mathcal X|\prod\limits_{S\in \mathcal X}\left(\frac{1}{\lambda_{D_k, S}}-1 \right)(-1)^{|\mathcal A|-|\mathcal X|}\prod\limits_{S\in \mathcal A\setminus \mathcal X}\frac{1}{\lambda_{D_k, S}}
\end{equation}

which simplifies to 

\begin{equation}
    \frac{1}{P(\mathcal S_+,\mathcal  S_-| \mathcal R)\prod\limits_{S\in \mathcal A}\lambda_{D_k, S}} \sum\limits_{\mathcal X\subseteq \mathcal A}|\mathcal X|(-1)^{|\mathcal A|-|\mathcal X|}\prod\limits_{S\in \mathcal X}\left(1-\lambda_{D_k, S}\right)
\end{equation}

To evaluate this term, define the function

\begin{equation}
    G(\mathcal A) := \sum\limits_{\mathcal X\subseteq \mathcal A}|\mathcal X|(-1)^{|\mathcal A|-|\mathcal X|}\prod\limits_{S\in \mathcal X}\left(1-\lambda_{D_k, S}\right)
\end{equation}

If we append an element $\{\tilde S\}$ to the set $\mathcal A$, where $\tilde S\not \in \mathcal A$, we can express $G(\mathcal A \cup \{\tilde S\})$ as

\begin{equation}
    G(\mathcal A\cup \{\tilde S\}) = \sum\limits_{\mathcal X\subseteq \mathcal A}|\mathcal X|(-1)^{|\mathcal A|+1-|\mathcal X|}\prod\limits_{S\in \mathcal X}\left(1-\lambda_{D_k, S}\right) + \sum\limits_{\mathcal X\subseteq \mathcal A}(|\mathcal X|+1)(-1)^{|\mathcal A|+1-|\mathcal X|-1}\prod\limits_{S\in \mathcal X}\left(1-\lambda_{D_k, S}\right)(1-\lambda_{D_k, \tilde S})
\end{equation}

where we have split the sum into subsets where containing $\tilde S$ and not containing $\tilde S$, and then expressed these in terms of the subsets $\mathcal X$ of $\mathcal A$. This yields the recursive formula 

\begin{equation}
    G(\mathcal A\cup \{\tilde S\}) = -\lambda_{D_k, \tilde S}G(\mathcal A) + (1-\lambda_{D_k, \tilde S})H(\mathcal A)
\end{equation}

where

\begin{equation}
    H(\mathcal A) = \sum\limits_{\mathcal X\subseteq \mathcal A}(-1)^{|\mathcal A|-|\mathcal X|}\prod\limits_{S\in \mathcal X}\left(1-\lambda_{D_k, S}\right)
\end{equation}

We can determine $H(\mathcal A)$ by the same technique -- noting that 

\begin{align*}
    H(\mathcal A \cup\{ \tilde S\}) &= \sum\limits_{\mathcal X\subseteq \mathcal A}(-1)^{|\mathcal A| + 1-|\mathcal X|}\prod\limits_{S\in \mathcal X}\left(1-\lambda_{D_k, S}\right) + \sum\limits_{\mathcal X\subseteq \mathcal A}(-1)^{|\mathcal A|+1-|\mathcal X|-1}\prod\limits_{S\in \mathcal X}\left(1-\lambda_{D_k, S}\right)(1-\lambda_{D_k, \tilde S})\\
    &= -H(\mathcal A) + (1-\lambda_{D_k, \tilde S})H(\mathcal A)\\
    &= -\lambda_{D_k, \tilde S}H(\mathcal A)
\end{align*}

for $\tilde S\not \in \mathcal A$. Then, noting that $H(\emptyset) = 1$, we recover 

\begin{equation}
    H(\mathcal A) = (-1)^{|\mathcal A|}\prod\limits_{S\in \mathcal A}\lambda_{D_k, S}
\end{equation}

and therefore 

\begin{align*}
    G(\mathcal A\cup \{\tilde S\}) &= -\lambda_{D_k, \tilde S}G(\mathcal A) + (1-\lambda_{D_k, \tilde S})(-1)^{|\mathcal A|}\prod\limits_{S\in \mathcal A}\lambda_{D_k, S}\\
    &= (-1)\left[\lambda_{D_k, \tilde S}G(\mathcal A) + (1-\lambda_{D_k, \tilde S})(-1)^{|\mathcal A\cup \{\tilde S\}|}\prod\limits_{S\in \mathcal A}\lambda_{D_k, S} \right]\\
\end{align*}

The above recursion relation states that for every new element we append to $\mathcal A$, we multiply the previous function by the new $\lambda_{D_k, \tilde S}$, add a term with the product of the previous $\lambda$'s multiplied by $(1-\lambda_{D_k, \tilde S})$, and multiply the result by $(-1)$. Starting from $G(\emptyset) = 0$ and $G(\{S\}) = 1-\lambda_{D_k, S}$, it follows that the function must take the form 

\begin{equation}
    G(\mathcal A) = (-1)^{|\mathcal A|+1}\sum\limits_{S\in \mathcal A} (1-\lambda_{D_k, S})\prod\limits_{S'\in \mathcal A\setminus S}\lambda_{D_k, S'}
\end{equation}

Therefore

\begin{align*}
    \mathbb E(D_k,\mathcal E)_\text{dis} &=\frac{1}{P(\mathcal S_+,\mathcal  S_-| \mathcal R)} \sum\limits_{\mathcal A\subseteq \mathcal S_+}\frac{1}{\prod\limits_{S\in \mathcal A}\lambda_{D_k, S}}(-1)^{|\mathcal A|+1}\sum\limits_{S\in \mathcal A} (1-\lambda_{D_k, S})\prod\limits_{S'\in \mathcal A\setminus S}\lambda_{D_k, S'}\,\, P(\mathcal S_- = 0, \mathcal A = 0, D_k = 1|\mathcal R)\\
    &= \frac{1}{P(\mathcal S_+,\mathcal  S_-| \mathcal R)} \sum\limits_{\mathcal A\subseteq \mathcal S_+}(-1)^{|\mathcal A|+1}P(\mathcal S_- = 0, \mathcal A = 0, D_k = 1|\mathcal R) \sum\limits_{S\in \mathcal A}\frac{1-\lambda_{D_k, S}}{\lambda_{D_k, S}}\numberthis
\end{align*}
Once again, we have arrived at a corrected form of the standard posterior 

\begin{equation}
    \mathbb E(D_k,\mathcal E)_\text{dis} = \frac{1}{P(\mathcal S_+,\mathcal  S_-| \mathcal R)} \sum\limits_{\mathcal A\subseteq \mathcal S_+}(-1)^{|\mathcal A|}P(\mathcal S_- = 0, \mathcal A = 0, D_k = 1|\mathcal R)\gamma (\mathcal A, D_k)
\end{equation}

where

\begin{equation}
    \gamma (\mathcal A, D_k) = \left|\mathcal A \right| - \sum\limits_{S\in \mathcal A}\frac{1}{\lambda_{D_k, S}}
\end{equation}

\begin{comment}
% old correction
\begin{equation}
    \gamma (\mathcal A, D_k) = \frac{\sum\limits_{S\in \mathcal A}(1-\lambda_{D_k, S})}{\prod\limits_{S\in \mathcal A}\lambda_{D_k, S}}
\end{equation}
\end{comment}
and we recover $\mathbb E(D_k,\mathcal E)_\text{dis} = P(D_k = 1 |\mathcal E)$ in the limit $\gamma (\mathcal A, D_k)\rightarrow 1$.

Finally, consider that for some $S\in \mathcal A$, $\lambda_{D_k, S} = 0$. Note that $P(\mathcal S_- = 0, \mathcal A = 0, D_k = 1|\mathcal R) = P(\mathcal S_- = 0, \mathcal A = 0|\mathcal R,  D_k = 1)P( D_k = 1|\mathcal R)$. If any $\lambda_{D_k, S} = 0$ for $S\in \mathcal S_-$, then this term is 0 by construction.

\end{proof}

\begin{comment}

\noindent \textbf{Theorem 2:} For three layer noisy-or networks, 

\begin{equation}
     \mathbb E(D_k,\mathcal E)_\text{dis} \equiv \mathbb E(D_k,\mathcal E)_\text{suff}
\end{equation}

\end{comment}

%\begin{proof}
%Follows immediately from Theorems 4 and 3. Note that this expression is identical to the expression given in Theorem 3 for the expected sufficiency. Hense, in three layer noisy-or models, the expected sufficiency and the expected disablement exactly coincide.
%\end{proof}

\section{Properties of the expected disablement}\label{appendix: properties exp}

In this appendix we show that the expected disablement satisfies our criteria for diagnostic measures. Although in noisy-or networks the expected disablement coincides with the expected sufficiency, which we have already shown to obey our postulates, we show here that the expected disablement in obeys our postulates in general models - regardless of the choice of graph topology or generative functions.

\begin{theorem}[Diagnostic properties of expected disablement]

The expected disablement, defined as 
\[
  \mathbb E(D_k,\mathcal E)_{\text{dis}} := \sum\limits_{\mathcal S'}\!\left|\mathcal S_+ \setminus \mathcal S_+' \right|P(\mathcal S' | \mathcal E, \text{do}(D_k = 0))
\]

satisfies the following three conditions

\begin{enumerate}
    \item{\emph{consistency}.}\quad $\mathbb E_\text{dis}(D_k, \mathcal E) \propto P(D_k = 1 |\mathcal E)$\\
    \item{\emph{causality}.}\quad If $\not \exists$ $S\in \mathsf{Dec}(D_k)\cap \mathcal S_+$ $\implies$ $\mathbb E_\text{dis}(D_k, \mathcal E) = 0$\\
    \item{\emph{simplicity}.}\quad  $\left|\mathbb E_\text{dis}(D_k, \mathcal E)\right| \leq \left|\mathcal S_+\cap \mathsf{Dec}(D_k) \right|$\\
\end{enumerate}
\end{theorem}

\begin{proof}
First we prove consistency. In the following, we use the notation $*$ to denote counterfactual variables. The term $P(\mathcal {S'}^* | \mathcal E, \text{do}(D^*_k = 0))$ can be expressed as 

\begin{align}
    P(\mathcal {S'}^* | \mathcal E, \text{do}(D^*_k = 0)) &= \sum\limits_{d_k\in \{0, 1\}}P(\mathcal {S'}^*, D_k = d_k | \mathcal E, \text{do}(D^*_k = 0))\\
    &= \sum\limits_{d_k\in \{0, 1\}}P(\mathcal {S'}^*| D_k = d_k , \mathcal E, \text{do}(D^*_k = 0))P( D_k = d_k | \mathcal E, \text{do}(D^*_k = 0))
\end{align}

As $D_k$ is not a descendent of $D_k^*$, this simplifies to 

\begin{equation}\label{silly expansion}
     P(\mathcal {S'}^* | \mathcal E, \text{do}(D^*_k = 0)) = \sum\limits_{d_k\in \{0, 1\}}P(\mathcal {S'}^*| D_k = d_k , \mathcal E, \text{do}(D^*_k = 0))P( D_k = d_k | \mathcal E)
\end{equation}

If $D_k = 0$ then the factual and counterfactual symptoms have identical states on their parents, and therefor are copies of each other. As a result, $\mathcal S_+ = \mathcal S_+'$ and the expected disablement is identical to 0. The only term that is non-zero is therefore when $D_k = 1$, and all non-zero terms in \eqref{silly expansion} therefore have a coefficient of $P(D_k = 1|\mathcal E)$. To see that causality is satisfied, note that $\left|\mathcal S_+ \setminus \mathcal S_+' \right|\neq 0$ iff $\mathcal S_+'\subset \mathcal S_+$, which requires that at least one symptom has been switched off. If $D_k$ is not a parent of any $\mathcal S_+$, then $ P(\mathcal {S'}^* | \mathcal E, \text{do}(D^*_k = 0)) = 0$ unless $\mathcal {S'}^* = \mathcal S$ (the symptom evidence is unchanged), which implies that $\left|\mathcal S_+ \setminus \mathcal S_+' \right|=0$, satisfying causality. Finally, note that $\mathbb E_\text{dis}(D_k,\mathcal E)$ is a convex combination over the values of the set difference function $\left|\mathcal S_+ \setminus \mathcal S_+' \right|$, and therefore is upper bounded by $\mathbb E_\text{dis}(D_k,\mathcal E)\leq \left|\mathcal S_+ \right|$, the number of positively evidenced symptoms that are children of $D_k$. Therefore, the expected disablement is upper bounded by the maximal number of positive symptoms that can be caused by $D_k$. 

\end{proof}

\section{Appendix of experimental results}\label{appendix: experimental}
In this appendix we list the results of experiments 1 and 2. Experiment 1 compares the top $k$ accuracy of our algorithms. In experiment 2 we compare the diagnostic accuracy of 44 doctors to our associative (Bayesian) and counterfactual diagnostic algorithms. The table below records the scores of each doctor and the associative and counterfactual algorithm shadowing them.\\

\bigskip 
\newpage

\begin{table}
\caption{Results for experiment 1: table shows the top $k$ accuracy for the posterior, expected disablement and expected sufficiency ranking algorithms, for $N$ from 1 to 15.}

\begin{tabular}{|c|c|c|c|}
\toprule
  N &            Posterior &          Disablement &          Sufficiency \\
\hline
  1 &  0.509 $ \pm $ 0.012 &  0.536 $ \pm $ 0.012 &  0.534 $ \pm $ 0.012 \\
  2 &  0.652 $ \pm $ 0.012 &  0.702 $ \pm $ 0.011 &  0.703 $ \pm $ 0.011 \\
  3 &  0.735 $ \pm $ 0.011 &   0.784 $ \pm $ 0.01 &   0.785 $ \pm $ 0.01 \\
  4 &   0.785 $ \pm $ 0.01 &  0.829 $ \pm $ 0.009 &  0.829 $ \pm $ 0.009 \\
  5 &  0.823 $ \pm $ 0.009 &  0.867 $ \pm $ 0.008 &   0.87 $ \pm $ 0.008 \\
  6 &  0.849 $ \pm $ 0.009 &  0.894 $ \pm $ 0.008 &  0.894 $ \pm $ 0.008 \\
  7 &  0.868 $ \pm $ 0.008 &   0.91 $ \pm $ 0.007 &   0.91 $ \pm $ 0.007 \\
  8 &  0.882 $ \pm $ 0.008 &  0.917 $ \pm $ 0.007 &  0.914 $ \pm $ 0.007 \\
  9 &  0.893 $ \pm $ 0.008 &  0.925 $ \pm $ 0.006 &  0.924 $ \pm $ 0.006 \\
 10 &  0.899 $ \pm $ 0.007 &   0.93 $ \pm $ 0.006 &  0.929 $ \pm $ 0.006 \\
 11 &  0.908 $ \pm $ 0.007 &  0.936 $ \pm $ 0.006 &  0.937 $ \pm $ 0.006 \\
 12 &  0.916 $ \pm $ 0.007 &  0.944 $ \pm $ 0.006 &  0.943 $ \pm $ 0.006 \\
 13 &  0.923 $ \pm $ 0.007 &  0.948 $ \pm $ 0.005 &  0.947 $ \pm $ 0.005 \\
 14 &  0.926 $ \pm $ 0.006 &  0.951 $ \pm $ 0.005 &   0.95 $ \pm $ 0.005 \\
 15 &  0.928 $ \pm $ 0.006 &  0.954 $ \pm $ 0.005 &  0.954 $ \pm $ 0.005 \\
 16 &  0.932 $ \pm $ 0.006 &  0.957 $ \pm $ 0.005 &  0.958 $ \pm $ 0.005 \\
 17 &  0.935 $ \pm $ 0.006 &  0.961 $ \pm $ 0.005 &  0.962 $ \pm $ 0.005 \\
 18 &  0.937 $ \pm $ 0.006 &  0.963 $ \pm $ 0.005 &  0.963 $ \pm $ 0.005 \\
 19 &  0.941 $ \pm $ 0.006 &  0.967 $ \pm $ 0.004 &  0.967 $ \pm $ 0.004 \\
 20 &  0.944 $ \pm $ 0.006 &  0.968 $ \pm $ 0.004 &  0.968 $ \pm $ 0.004 \\
\hline
\end{tabular}

\end{table}

\begin{table}[!h]
\footnotesize
\caption{Results for experiment 1: table shows the mean position of the true disease for the associative (A) and counterfactual (C, expected sufficiency) algorithms over all 1671 cases. Results are stratified over the rareness of the disease (given the age and gender of the patient). For each disease rareness category, the number of cases N is given. Also the number of cases where the associative algorithm ranked the true disease higher than the counterfactual algorithm (Wins (A)), the counterfactual algorithm ranked the true disease higher than the associative algorithm (Wins (C)), and the number of cases where the two algorithms ranked the true disease in the same position (Draws) are given, for all cases and for each disease rareness class.}
\label{tab:topn counts} 
\begin{center}
\begin{tabular}{| c | c | c | c | c | c | c | c|}
\hline
\textbf{} & \multicolumn{7}{ c |}{\textbf{Vignettes}}  \\ 
\cline{2-8}
 & \textbf{All} & \textbf{Very common} & \textbf{Common} & \textbf{Uncommon} & \textbf{Rare} & \textbf{Very rare} & \textbf{Almost Impossible}\\
\hline
N & 1671 & 131 & 413 & 546 & 353 & 210 & 18 \\ \hline
Mean position (A)  & 3.81 $\pm$ 5.25  & 2.85 $\pm$ 4.27 & 2.71 $\pm$ 3.86 & 3.72 $\pm$ 5.05 & 4.35 $\pm$ 5.28 & 5.45 $\pm$ 6.52 & 4.22 $\pm$ 5.19 \\ \hline
Mean position (C)  & 3.16 $\pm$ 4.40  & 2.5 $\pm$ 3.55 & 2.32 $\pm$ 3.25 & 3.01 $\pm$ 4.07 & 3.72 $\pm$ 4.74 & 4.38 $\pm$ 5.53 & 3.56 $\pm$ 3.96\\ \hline
Wins (A)& 31 &  2  &  7 & 9 & 9 & 4 & 0\\ \hline
Wins (C)  & 412 & 20 & 80 & 135 & 103 & 69 & 5\\ \hline
Draws & 1228 & 131 & 326 & 402 & 241 & 137 & 13\\ \hline
\end{tabular}
\end{center}
\end{table}

\begin{table}
\caption{Results for experiment 2: table shows the accuracy obtained by the doctor and each algorithm shadowing the doctors, for each of the 44 single-doctor experiments. The accuracies are reported with the standard standard deviation of the mean estimator.}
\label{tab:exp_suff} 
\begin{tabular}{|c|c|c|c|c|}
\hline
Doctor number &      Doctor accuracy &            Posterior & Expected sufficiency & Expected disablement \\
\hline
0  &  0.725 $ \pm $ 0.019 &   0.656 $ \pm $ 0.02 &  0.694 $ \pm $ 0.019 &  0.692 $ \pm $ 0.019 \\
1  &  0.823 $ \pm $ 0.022 &  0.719 $ \pm $ 0.026 &  0.771 $ \pm $ 0.025 &  0.774 $ \pm $ 0.025 \\
2  &   0.89 $ \pm $ 0.018 &  0.791 $ \pm $ 0.023 &  0.834 $ \pm $ 0.021 &  0.837 $ \pm $ 0.021 \\
3  &  0.805 $ \pm $ 0.023 &  0.811 $ \pm $ 0.023 &  0.834 $ \pm $ 0.021 &  0.831 $ \pm $ 0.022 \\
4  &  0.776 $ \pm $ 0.034 &  0.855 $ \pm $ 0.029 &  0.908 $ \pm $ 0.023 &  0.914 $ \pm $ 0.023 \\
5  &  0.612 $ \pm $ 0.028 &  0.779 $ \pm $ 0.024 &  0.827 $ \pm $ 0.022 &  0.834 $ \pm $ 0.021 \\
6  &   0.799 $ \pm $ 0.02 &  0.739 $ \pm $ 0.022 &   0.794 $ \pm $ 0.02 &   0.794 $ \pm $ 0.02 \\
7  &  0.778 $ \pm $ 0.026 &  0.767 $ \pm $ 0.026 &  0.825 $ \pm $ 0.024 &  0.825 $ \pm $ 0.024 \\
8  &   0.69 $ \pm $ 0.025 &  0.788 $ \pm $ 0.022 &   0.833 $ \pm $ 0.02 &   0.833 $ \pm $ 0.02 \\
9  &  0.698 $ \pm $ 0.058 &   0.81 $ \pm $ 0.049 &  0.873 $ \pm $ 0.042 &  0.873 $ \pm $ 0.042 \\
10 &  0.905 $ \pm $ 0.037 &  0.841 $ \pm $ 0.046 &  0.873 $ \pm $ 0.042 &   0.889 $ \pm $ 0.04 \\
11 &  0.783 $ \pm $ 0.034 &   0.72 $ \pm $ 0.038 &  0.797 $ \pm $ 0.034 &  0.797 $ \pm $ 0.034 \\
12 &  0.684 $ \pm $ 0.053 &    0.75 $ \pm $ 0.05 &  0.789 $ \pm $ 0.047 &  0.776 $ \pm $ 0.048 \\
13 &  0.627 $ \pm $ 0.063 &  0.712 $ \pm $ 0.059 &   0.78 $ \pm $ 0.054 &   0.78 $ \pm $ 0.054 \\
14 &  0.788 $ \pm $ 0.033 &  0.737 $ \pm $ 0.035 &  0.776 $ \pm $ 0.033 &  0.782 $ \pm $ 0.033 \\
15 &  0.891 $ \pm $ 0.018 &   0.73 $ \pm $ 0.025 &  0.776 $ \pm $ 0.024 &  0.776 $ \pm $ 0.024 \\
16 &  0.791 $ \pm $ 0.043 &  0.835 $ \pm $ 0.039 &  0.879 $ \pm $ 0.034 &  0.879 $ \pm $ 0.034 \\
17 &  0.651 $ \pm $ 0.051 &  0.767 $ \pm $ 0.046 &  0.802 $ \pm $ 0.043 &  0.802 $ \pm $ 0.043 \\
18 &  0.722 $ \pm $ 0.043 &  0.806 $ \pm $ 0.038 &  0.833 $ \pm $ 0.036 &  0.833 $ \pm $ 0.036 \\
19 &   0.75 $ \pm $ 0.056 &  0.717 $ \pm $ 0.058 &  0.767 $ \pm $ 0.055 &  0.783 $ \pm $ 0.053 \\
20 &  0.566 $ \pm $ 0.068 &  0.642 $ \pm $ 0.066 &   0.66 $ \pm $ 0.065 &   0.66 $ \pm $ 0.065 \\
21 &  0.797 $ \pm $ 0.026 &   0.73 $ \pm $ 0.029 &  0.776 $ \pm $ 0.027 &  0.781 $ \pm $ 0.027 \\
22 &   0.671 $ \pm $ 0.03 &   0.667 $ \pm $ 0.03 &  0.736 $ \pm $ 0.028 &  0.735 $ \pm $ 0.028 \\
23 &  0.695 $ \pm $ 0.032 &   0.67 $ \pm $ 0.033 &  0.709 $ \pm $ 0.032 &  0.708 $ \pm $ 0.032 \\
24 &  0.735 $ \pm $ 0.035 &   0.71 $ \pm $ 0.036 &  0.781 $ \pm $ 0.033 &  0.774 $ \pm $ 0.034 \\
25 &  0.648 $ \pm $ 0.047 &  0.705 $ \pm $ 0.045 &  0.752 $ \pm $ 0.042 &  0.752 $ \pm $ 0.042 \\
26 &    0.7 $ \pm $ 0.065 &   0.66 $ \pm $ 0.067 &   0.66 $ \pm $ 0.067 &   0.66 $ \pm $ 0.067 \\
27 &  0.854 $ \pm $ 0.035 &  0.777 $ \pm $ 0.041 &  0.835 $ \pm $ 0.037 &  0.835 $ \pm $ 0.037 \\
28 &  0.787 $ \pm $ 0.039 &   0.778 $ \pm $ 0.04 &  0.824 $ \pm $ 0.037 &  0.815 $ \pm $ 0.037 \\
29 &  0.636 $ \pm $ 0.048 &  0.697 $ \pm $ 0.046 &  0.747 $ \pm $ 0.044 &  0.747 $ \pm $ 0.044 \\
30 &  0.604 $ \pm $ 0.046 &  0.739 $ \pm $ 0.042 &  0.748 $ \pm $ 0.041 &  0.748 $ \pm $ 0.041 \\
31 &  0.758 $ \pm $ 0.053 &  0.818 $ \pm $ 0.047 &  0.909 $ \pm $ 0.035 &  0.908 $ \pm $ 0.036 \\
32 &  0.825 $ \pm $ 0.039 &  0.691 $ \pm $ 0.047 &  0.711 $ \pm $ 0.046 &  0.701 $ \pm $ 0.046 \\
33 &    0.5 $ \pm $ 0.065 &  0.633 $ \pm $ 0.062 &   0.683 $ \pm $ 0.06 &   0.683 $ \pm $ 0.06 \\
34 &  0.607 $ \pm $ 0.063 &  0.607 $ \pm $ 0.063 &  0.689 $ \pm $ 0.059 &  0.689 $ \pm $ 0.059 \\
35 &  0.574 $ \pm $ 0.063 &  0.623 $ \pm $ 0.062 &  0.689 $ \pm $ 0.059 &  0.689 $ \pm $ 0.059 \\
36 &   0.55 $ \pm $ 0.064 &  0.633 $ \pm $ 0.062 &  0.667 $ \pm $ 0.061 &  0.667 $ \pm $ 0.061 \\
37 &   0.61 $ \pm $ 0.063 &  0.576 $ \pm $ 0.064 &  0.661 $ \pm $ 0.062 &  0.661 $ \pm $ 0.062 \\
38 &   0.592 $ \pm $ 0.04 &  0.697 $ \pm $ 0.037 &  0.724 $ \pm $ 0.036 &  0.715 $ \pm $ 0.037 \\
39 &  0.708 $ \pm $ 0.044 &   0.67 $ \pm $ 0.046 &  0.717 $ \pm $ 0.044 &  0.708 $ \pm $ 0.044 \\
40 &  0.702 $ \pm $ 0.045 &  0.721 $ \pm $ 0.044 &   0.74 $ \pm $ 0.043 &   0.74 $ \pm $ 0.043 \\
41 &  0.765 $ \pm $ 0.059 &  0.765 $ \pm $ 0.059 &  0.824 $ \pm $ 0.053 &  0.824 $ \pm $ 0.053 \\
42 &  0.639 $ \pm $ 0.053 &  0.723 $ \pm $ 0.049 &  0.783 $ \pm $ 0.045 &  0.768 $ \pm $ 0.047 \\
43 &  0.704 $ \pm $ 0.054 &  0.648 $ \pm $ 0.057 &  0.704 $ \pm $ 0.054 &  0.704 $ \pm $ 0.054 \\
\hline
\end{tabular}

\end{table}

\end{appendices}

\end{document}